\documentclass[11pt]{article}

\usepackage{tikz}
\usetikzlibrary{decorations.pathreplacing}
\usetikzlibrary{math}
\usepackage{pgfplots}

\usepackage[utf8]{inputenc} 
\usepackage[T1]{fontenc}    
\usepackage{lmodern}
\usepackage[table]{xcolor}

\input{setup.sty}
\newcommand{\bysest}{\what{\theta}_{\mathsf{B}}}
\newcommand{\trerr}{\mathsf{Train}}
\newcommand{\prerr}{\mathsf{Pred}}
\newcommand{\cost}{\mathsf{Cost}}
\newcommand{\model}{\mathscr{M}}
\newcommand{\Cov}{\mathsf{Cov}}
\newcommand{\de}{\mathrm{d}}

\newcommand{\tfunc}{\mathsf{Train}}
\newcommand{\jfunc}{\mathsf{J}}
\newcommand{\ifunc}{\mathsf{I}}
\newcommand{\mmse}{\mathsf{mmse}}

\usepackage{enumitem}
\usepackage{natbib}
\usepackage{mathrsfs}
\usepackage{setspace}
\usepackage{xr}

\title{Is Memorization Helpful or Harmful? \\ Prior Information Sets the Threshold}

\newcommand*\samethanks[1][\value{footnote}]{\footnotemark[#1]}

\author{%
	Chen Cheng\thanks{Department of Statistics, University of Chicago} \and Rina Foygel Barber\samethanks
}

\begin{document}
\maketitle

\begin{abstract}
    We examine the connection between training error and generalization error for arbitrary estimating procedures, working in an overparameterized linear model under general priors in a Bayesian setup. We find determining factors inherent to the prior distribution $\pi$, giving explicit conditions under which optimal generalization necessitates that the training error be (i) near interpolating relative to the noise size (i.e., memorization is necessary), or (ii) close to the noise level (i.e., overfitting is harmful). Remarkably, these phenomena occur when the noise reaches thresholds determined by the Fisher information and the variance parameters of the prior $\pi$.
    
\end{abstract}


\section{Introduction}

In this work, we consider the question of \emph{overfitting} in high-dimensional statistical models: should we allow our trained model to overfit to the training data, perhaps even interpolating the data perfectly, or should we instead constrain the training process to return low-complexity models that avoid overfitting? This core question lies at the heart of statistics and machine learning in practice, and has received substantial attention in the theoretical literature as well, over several decades. Despite its importance and ubiquity, however, this question is still not fully understood: there are a range of findings that are apparently at odds with each other.
In classical statistical machine learning, it is standard to train a model that exhibits some intrinsic low-dimensional structure, such as sparsity or low-rank, to ensure that we avoid overfitting~\citep{tibshirani1996regression, candes2008enhancing, candes2012exact}.
On the other hand, recent work by \citet{cheng2022memorize} suggests that, in an overparameterized regime, memorization is \emph{necessary} for generalization---that is, it is necessary to choose a model that overfits, in order to achieve optimal prediction error. 
And, another line of literature on the ``benign overfitting'' phenomenon can be viewed as lying in between these two extremes, establishing that overfitting is not harmful (but, perhaps not necessary) for good predictive performance \citep{belkin2018overfitting, belkin2019does, bartlett2020benign, hastie2022surprises, cheng2024dimension}. How can we reconcile these different views?

We examine this question from the perspective of a Bayesian framework, in the context of a specific problem: high-dimensional linear regression. Concretely, we consider the model
\[y = X\theta + \normal(0,\sigma^2 I_n),\]
where $\theta\in\R^d$ is drawn from a prior $\pi$, and where $d\geq n$ (the overparameterized regime).
In this model, an overfitted estimator $\what{\theta}$ is one that has (mean square) training error substantially below the noise level $\sigma^2$.
At a high level, our findings reveal that the question of whether overfitting is necessary or is instead harmful for optimal predictive peformance depends entirely on the nature of the prior $\pi$. If $\pi$ encodes some sort of latent low-dimensional structure (for instance, encouraging approximately sparse $\theta$) then overfitting is harmful, while if $\pi$ is uninformative then overfitting becomes necessary. Strikingly, our findings suggest the (asymptotically optimal) determining factors delineating these scenarios for the prior $\pi$ with a differentiable density $p>0$ are the Fisher information and variance parameters:\footnote{We note that $\jfunc_{\pi}$ denotes the Fisher information for $\pi$ alone~\cite[Sec.~C]{dembo2002information}, while the classical definition is for a parametric family of distributions $\{p_\xi(\cdot)\}_{\xi \in \R^n}$. Our definition is compatible with the Fisher information evaluated at $\xi = 0$ for the location model by $p_\xi(\cdot) = p(\cdot + \xi)$.}
\begin{subequations}
\begin{align}
    & \text{Fisher information:} & &   \jfunc_\pi =  \E \brk{\frac{1}{n}\norm{\nabla \log p(X\theta)}^2\mid X} , \label{eq:def-Fisher-parameter} \\
    & \text{Variance parameter:} & & \mathsf{V}_\pi =  \E\brk{\frac{1}{n}\norm{X\theta}^2\mid X}, \label{eq:def-variance-parameter}
\end{align}
\end{subequations}
where in order to define variance, from this point on we assume that $\pi$ has mean zero, without loss of generality.
Intuitively, $\jfunc_\pi$ measures spikiness of the prior $\pi$---and thus $\jfunc_\pi^{-1}$ acts as a local measure of effective dimension. In contrast, $\mathsf{V}_\pi$ measures a macroscopic notion of dimension of $\pi$ by computing a variance. Our key message is that comparing the noise level $\sigma^2$ to $\jfunc_\pi^{-1}$ determines when it is necessary to overfit, and comparing to $\mathsf{V}_\pi$ determines when it is harmful to overfit.

Indeed, the Fisher information and variance parameters depend directly on the push-forward distribution of $X\theta$. As we primarily
treat the design matrix $X$ as fixed, and therefore the prior $\pi$ on $\theta$ directly specifies a distribution on
$X\theta$, we refer to quantities depend on either the distribution of $X \theta$ or $\theta$ as information about the prior.

\subsection{Related work}
High-dimensional statistical models have played a central role in extending classical large-sample asymptotic theory to modern overparametrization settings. Pioneering work that exploits intrinsic low-rank and sparsity structures in data typically relies on explicit regularization mechanisms, such as $\ell_1$-penalization~\citep{tibshirani1996regression} in compressed sensing~\citep{candes2008enhancing} and nuclear-norm regularization in matrix completion~\citep{candes2012exact}.

The remarkable success of deep learning over the past two decades questions the presumed necessity of explicit regularization for the accurate recovery of low-rank structures. In particular, the phenomenon of ``implicit regularization''—whereby specific optimization procedures converge to good local minima exhibiting ``benign overfitting''—has been the subject of extensive investigation, especially in the context of large-scale datasets and deep neural networks~\citep{neyshabur2017implicit, gunasekar2017implicit}. Closely related to our work, \citet{hastie2022surprises} analyze the minimum-$\ell_2$-norm interpolating solution in high-dimensional linear regression, with subsequent extensions to broader settings including kernel regression~\citep{liang2020just}, random feature models~\citep{mei2022generalization}, and general Hilbert space frameworks~\citep{bartlett2020benign, cheng2024dimension}.

Memorization, or the necessity of interpolating models, still remains largely underexplored. \citet{feldman2020does} initiate the study of this question in the multi-class classification setting, where the data-generating distribution is modeled as a mixture of heavy-tailed subpopulations and, consequently, does not exhibit low-rank structure. \citet{cheng2022memorize} investigate the necessity of interpolation in overparameterized linear models, and here we extend their framework to (almost) arbitrary priors. Characterizing the fundamental limits of model capacity making interpolation necessary has also emerged in several recent works along this line of research~\citep{shah2025does, muller2025all}.

We finally point out that our analyses, largely inspired by prior research on information-theoretic inequalities~\citep{stam1959some, dembo2002information}, deeply connect to the well established results on the minimum mean square error (MMSE) in Gaussian channels~\citep{guo2004mutual, guo2011estimation, ledoux2016heat}.

\section{Main results} \label{sec:main}
\subsection{Problem setup}\label{sec:problem_setup}
We first specify our problem setting and introduce the corresponding quantities of interest.

\paragraph{Overparameterized linear model under a general prior $\pi$.} Consider a design matrix $X\in\R^{n\times d}$,  whose rows are vectors $x_1^\top, \cdots, x_n^\top \in \R^d$. The design matrix may be fixed or random, as we carry out all our calculations conditional on $X$. We assume that $d\geq n$ (the overparameterized regime), and that $X$ has full row rank (almost surely). 

Given a prior $\pi$ on $\R^d$ and a noise level $\sigma^2>0$, we assume the following model for the observed response vector $y\in\R^n$:
\[\textnormal{The Bayesian model $\model_X(\pi,\sigma^2)$: \ $y = X\theta + \sigma \tau$ where $(\theta,\tau)\sim \pi \times \normal(0, I_n)$.}\]
That is, this is a Gaussian linear model, with coefficients $\theta\in\R^d$ sampled from the prior $\pi$.

\paragraph{Training error and prediction error.}
We study estimating $\theta$ of the linear model given the design matrix $X$ and the observation $y$. Given an estimator $\what{\theta} = \what{\theta}(X,y)$, we define its expected training error as
\[\trerr(\what{\theta}) = \E \brk{\frac{1}{n} \norm{X \what{\theta} - y}^2 \mid X},\]
where implicitly, the expected value is computed with respect to the model $\model_X(\pi,\sigma^2)$. The training error will typically take values in the range $[0,\sigma^2]$. Indeed,
any estimator $\what{\theta}$ that interpolates the data will have $\trerr(\what{\theta})=0$---that is, perfect memorization. On the other hand, an estimator with zero overfitting will have $\trerr(\what{\theta})=\sigma^2$ (i.e., the true noise level).

We can also define its prediction error at a new test point $x'$ as $\E\brk{(x'{}^\top\what{\theta} - x'{}^\top \theta)^2\mid X}$. In particular, if $x'$ is random and mean-zero, and is independent from the training data $y$, we can equivalently define the prediction error as
\[\prerr_\Sigma(\what{\theta}) = \E\brk{\norm{\what{\theta} - \theta}_{\Sigma}^2\mid X},\]
where $\Sigma = \textnormal{Cov}(x')$ is the covariance of the random test point, and where $\|v\|_{\Sigma}:=\sqrt{v^\top \Sigma v}$. (If $\Sigma = I_d$, then this is simply the mean-squared error (MSE) of the estimator $\what{\theta}$.)
We will also write
\[\cost(\what{\theta}) = \prerr_\Sigma(\what{\theta}) - \prerr_\Sigma^*\textnormal{ \ where \ }\prerr_\Sigma^* = \inf_{\what{\theta}} \prerr_\Sigma(\what{\theta}),\]
i.e., this measures the excess prediction error of the estimator $\what{\theta}$.\footnote{Since $X$ is treated as fixed (i.e., we condition on $X$ throughout), we should interpret the definition of $\prerr_\Sigma^*$ as computing an infimum over all possible maps $y\mapsto \what{\theta}(X,y)$, i.e., functions $\R^n\to\R^d$. As a technical note, we will assume throughout that any estimator $\what{\theta}$ considered in this paper is square-integrable (with respect to the distribution of its input, the observed response $y$), so that all defined measures of training and prediction error are finite.}

The key question of this paper is the following: what is the relationship between the training error and the prediction error of an estimator---and, to ensure optimal predictive performance (i.e., minimal prediction error), what training error should we aim for? 
In particular, we are interested in identifying settings where these interesting phenomena may occur:
\begin{itemize}
    \item Settings where \textbf{memorization is necessary} for generalization: when is it true that any estimator $\what{\theta}$ that is (near) optimal for predictive error, must have training error close to zero, i.e.,
    \[\cost_\Sigma(\what{\theta})\approx  0 \ \Longrightarrow \ \trerr(\what{\theta}) \leq o(\sigma^2) \ ?\]
    \item Settings where \textbf{overfitting is harmful} for generalization: when is it true that any estimator $\what{\theta}$ that is (near) optimal for predictive error, must have training error close to $\sigma^2$, i.e.,
    \[\cost_\Sigma(\what{\theta}) \approx 0 \ \Longrightarrow \ \trerr(\what{\theta}) \geq \sigma^2 - o(\sigma^2) \ ?\]
\end{itemize}


\subsection{Comparing to the Bayes-optimal estimator}
Under our assumption of the Bayesian model $\model_X(\pi,\sigma^2)$, the Bayes estimator is given by the posterior mean,
\[\bysest(X,y) = \E [\theta \mid X, y].\]
The following proposition demonstrates that $\bysest$ plays a crucial role for the motivating questions of this paper.
\begin{proposition}\label{prop:compare_to_Bayes}
    Under the setting and notation above, for any positive definite $\Sigma\in\R^{d\times d}$, $\bysest$ achieves the optimal prediction error, i.e.,
\[\prerr_\Sigma(\bysest) = \inf_{\what{\theta}}\prerr_\Sigma(\what{\theta}) = \prerr_\Sigma^*.\]
Moreover, letting $\lambda_\Sigma = \frac{1}{n}\|X \Sigma^{-\half}\|^2$, for any estimator $\what{\theta}$, its excess prediction error satisfies
\[\cost_\Sigma(\what{\theta}) = \prerr_\Sigma(\what{\theta}) - \prerr_\Sigma^* = \E \brk{\norm{\what{\theta} - \bysest}_{\Sigma}^2\mid X} \geq \lambda_\Sigma^{-1}\prn{\sqrt{\trerr(\what{\theta})} - \sqrt{\trerr(\bysest)}}^2.\]
\end{proposition}
In other words, this standard result tells us that the Bayes estimator is optimal for prediction error, regardless of $\Sigma$. Moreover, the cost of choosing a different estimator (in terms of excess prediction error) can be lower bounded via the difference in training errors, $\trerr(\what{\theta})$ versus $\trerr(\bysest)$. In particular, any estimator $\what{\theta}$ that is close to optimal in terms of prediction error, must have a training error that is similar to that of the Bayes estimator:
\[\cost_\Sigma(\what{\theta}) \approx 0 \ \Longrightarrow \ \trerr(\what{\theta})\approx \trerr(\bysest).\]

\subsection{Training error of the Bayes estimator}
The result of Proposition~\ref{prop:compare_to_Bayes} suggests that, to characterize the regimes in which memorization is necessary or in which overfitting is harmful, we need to examine the training error $\trerr(\bysest)$ of the Bayes estimator. From this point on, we assume that $\pi$ has a positive and differentiable density, and we write $p$ as the density of $X\theta$ induced by $\theta\sim \pi$ (with $X$ treated as fixed). Recall the Fisher information (cf.~Eq.~\eqref{eq:def-Fisher-parameter}) and variance (cf.~Eq.~\eqref{eq:def-variance-parameter}) parameters $\jfunc_\pi$ and $\mathsf{V}_\pi$, which we assume are positive and finite in this section.
In particular, we note an important and well-known property:\footnote{Essentially this is the classical Cram\'{e}r-Rao bound. See this result in \cite[Eq.~(2.1)]{stam1959some}. Indeed, the upper bound by $\mathsf{V}_\pi$ can be sharpened into the entropy power. We refer the reader to Lemma~\ref{lem:EII} for details. }
\[\textnormal{For any $\pi$, it holds that $\mathsf{V}_\pi\geq \jfunc_\pi^{-1}$, with equality if and only if $X\theta\sim \normal(0,\nu^2 I_n)$ for some $\nu^2>0$.}\]

The following result characterizes the training error of $\bysest$ in terms of these two parameters.
\begin{theorem}\label{thm:Bayes_train_err}
Fix any $d\geq n\geq 1$. Let $\pi$ be any prior with positive and differentiable density, and let $X\in\R^{n\times d}$ have full row rank $n$. Then under the model $\model_X(\pi,\sigma^2)$, the training error of the Bayes estimator satisfies
    \begin{equation}\label{eqn:sandwich_ineq_thm:Bayes_train_err}\frac{\sigma^4}{\mathsf{V}_{\pi} + \sigma^2} \leq \trerr(\bysest) \leq  \frac{\sigma^4}{\jfunc_{\pi}^{-1} + \sigma^2}\textnormal{ \ for all $\sigma^2>0$}.\end{equation}
Moreover, the training error approaches the above upper bound for vanishing noise,
\begin{equation}\label{eqn:asymp_low_noise_thm:Bayes_train_err}\trerr(\bysest) = \frac{\sigma^4}{\jfunc_{\pi}^{-1} + \sigma^2} + o(\sigma^4) = \jfunc_{\pi}\sigma^4  + o(\sigma^4) \textnormal{ \ as \ }\sigma^2 \to 0,\end{equation}
and approaches the above lower bound for increasing noise,
\begin{equation}\label{eqn:asymp_high_noise_thm:Bayes_train_err}\trerr(\bysest) = \frac{\sigma^4}{\mathsf{V}_{\pi} + \sigma^2} + o(1) = \sigma^2 - \mathsf{V}_{\pi} + o(1) \textnormal{ \ as \ }\sigma^2 \to \infty.\end{equation}
\end{theorem}
\noindent The asymptotic statements~\eqref{eqn:asymp_low_noise_thm:Bayes_train_err} and~\eqref{eqn:asymp_high_noise_thm:Bayes_train_err} can be interpreted as follows. With vanishing noise $\sigma^2\to 0$, the Bayes estimator exhibits memorization (since its training error is $\trerr(\bysest) = O(\sigma^4) = o(\sigma^2)$. In contrast, when $\sigma^2\to\infty$, the Bayes estimator instead shows a limited amount of overfitting, since $\sigma^2 - \trerr(\bysest) = \mathsf{V}_\pi + o(1) = o(\sigma^2)$.\footnote{We emphasize that in the asymptotic claims~\eqref{eqn:asymp_low_noise_thm:Bayes_train_err} and~\eqref{eqn:asymp_high_noise_thm:Bayes_train_err}, the statements hold for fixed $d,n$, fixed design matrix $X$, and a fixed choice of the prior $\pi$; the $o(\cdot)$ terms reveal the asymptotic dependence on the noise level $\sigma^2$ only.}

\subsection{The tradeoff between training error and prediction error}\label{sec:main_results_tradeoff}
With the results above in place, we are now ready to examine the tradeoff between training error and prediction error, for any estimator $\what{\theta}$.
\begin{theorem}\label{thm:main}
    Fix any $d\geq n\geq 1$. Let $\pi$ be any prior with positive and differentiable density, and let $X\in\R^{n\times d}$ have full row rank $n$. Fix any positive definite $\Sigma\in\R^{d\times d}$. Then under the model $\model_X(\pi,\sigma^2)$, for any $\sigma^2>0$ and any estimator $\what{\theta}$, 
    \begin{equation}\label{eqn:trerr_too_high__thm:main}\textnormal{If $\trerr(\what{\theta})\geq \frac{\sigma^4}{\jfunc_\pi^{-1}+\sigma^2}$ \ then \ }\cost_\Sigma(\what{\theta})\geq \lambda_\Sigma^{-1}\prn{\sqrt{\trerr(\what{\theta})}-\sqrt{\frac{\sigma^4}{\jfunc_\pi^{-1}+\sigma^2}}}^2,\end{equation}
    and
    \begin{equation}\label{eqn:trerr_too_low__thm:main}\textnormal{If $\trerr(\what{\theta})\leq \frac{\sigma^4}{\mathsf{V}_\pi+\sigma^2}$ then }\cost_\Sigma(\what{\theta})\geq \lambda_\Sigma^{-1}\prn{\sqrt{\frac{\sigma^4}{\mathsf{V}_\pi+\sigma^2}} - \sqrt{\trerr(\what{\theta})}}^2.\end{equation}
\end{theorem}
\begin{proof}
    These results follow immediately from combining the bounds~\eqref{eqn:sandwich_ineq_thm:Bayes_train_err} in Theorem~\ref{thm:Bayes_train_err} (which calculates the training error of $\bysest$), with the result of Proposition~\ref{prop:compare_to_Bayes} (which tells us that optimal prediction error can only be achieved by an estimator whose training error matches that of $\bysest$).
\end{proof}

The two parts of this theorem immediately yield answers to our two key questions: under what regimes does it hold that \textbf{memorization is necessary}, or, that \textbf{overfitting is harmful}? The following two corollaries answer these questions.

\begin{figure}
    \centering
    \begin{tikzpicture}

\tikzmath{\ytick=0.15;}
\tikzmath{\xstart0 = 0; \xstart1 = 8; \xend0 = 5; \xend1 = 13; \xmid0 = 3.05; \xmid1 = 11.05; \xpoint0 = 1; \xpointnext0 = 2; \xpoint1 = 11.9; \xpointnext1 = 10.8;}
    \tikzmath{\xbrace = 0.1; \ybrace=0.1;}
\tikzmath{\xtexthi0 = 4.125; \xtexthinext0 = 4; \xtextlo1 = 9.95; \xtextlonext1 = 10.4;}

\draw[thick,->] (\xstart0,0) -- (\xend0+0.5,0);
\draw[thick,->] (\xstart1,0) -- (\xend1+0.5,0);
\draw[thick] (\xstart0,-\ytick) -- (\xstart0,-\ytick) node[anchor=north]{\small $0$};
\draw[thick,->] (\xstart0,0) -- (\xstart0,3);
\draw[thick,->] (\xstart1,0) -- (\xstart1,3);
\draw[thick] (\xend0,\ytick) -- (\xend0,-\ytick) node[anchor=north]{\small $\sigma^2$};
\draw[thick] (\xpoint0,\ytick) -- (\xpoint0,-\ytick) node[anchor=north]{\small $\jfunc_\pi \sigma^4$};
\draw[thick] (\xpointnext0,\ytick) -- (\xpointnext0,-\ytick) node[anchor=north]{\small $C \jfunc_\pi \sigma^4$};
\draw[thick] (\xstart1,-\ytick) -- (\xstart1,-\ytick) node[anchor=north]{\small $0$};
\draw[thick] (\xend1,\ytick) -- (\xend1,-\ytick) node[anchor=north] {\small $\sigma^2$};
\draw[thick] (\xpoint1,\ytick) -- (\xpoint1,-\ytick) node[anchor=north] {\small \quad $\sigma^2-\mathsf{V}_\pi$};
\draw[thick] (\xpointnext1,\ytick) -- (\xpointnext1,-\ytick) node[anchor=north] {\small $\sigma^2-C\mathsf{V}_\pi$\quad\quad};
\node[align=center] at (\xend0+1,-0.25) {\small $\trerr(\what{\theta})$};
\node[align=center] at (\xend1+1,-0.25) {\small $\trerr(\what{\theta})$};
\node[align=center] at (\xmid0,4) {Memorization is necessary when $\sigma^2 \leq \jfunc_\pi^{-1}$};
\node[align=center] at (\xmid1,4) {Overfitting is harmful when $\sigma^2 \geq \mathsf{V}_\pi$};

\draw [decorate,decoration={brace,amplitude=10pt}, thick](\xend0-\xbrace,-\ybrace-0.7) -- (\xpointnext0+\xbrace,-\ybrace-0.7);
\draw [decorate,decoration={brace,amplitude=10pt}, thick](\xpoint0+\xbrace,1.8+\ybrace) -- (\xend0-\xbrace,1.8+\ybrace);
\draw [decorate,decoration={brace,amplitude=10pt}, thick](\xstart1+\xbrace,1.8+\ybrace) -- (\xpoint1-\xbrace,1.8+\ybrace);
\draw [decorate,decoration={brace,amplitude=10pt}, thick](\xpointnext1-\xbrace,-\ybrace-0.7) -- (\xstart1+\xbrace,-\ybrace-0.7);

\node[align=center] at (\xtexthi0,2.7) {\small \begin{tabular}{c}any $\what{\theta}$ in this range\\ has $\cost_\Sigma(\what{\theta})>0$\end{tabular}};

\node[align=center] at (\xtextlo1,2.7) {\small \begin{tabular}{c}any $\what{\theta}$ in this range\\ has $\cost_\Sigma(\what{\theta})>0$\end{tabular}};

\node[align=center] at (\xtexthinext0,-1.6) {\small \begin{tabular}{c}The lower bound on $\cost_\Sigma(\what{\theta})$\\ grows linearly in this range\end{tabular}};

\node[align=center] at (\xtextlonext1,-1.6) {\small \begin{tabular}{c}The lower bound on $\cost_\Sigma(\what{\theta})$\\ grows quadratically in this range\end{tabular}};

\draw[draw=black] (\xstart0-0.7,-2.2) rectangle (\xend0+1.8,3.5);
\draw[draw=black] (\xstart1-0.7,-2.2) rectangle (\xend1+1.8,3.5);

\draw [red] plot [smooth] coordinates {(0,0.1) (0.4,0.025) (0.8,0.01) (1.2,0.025) (1.6,0.1) (2.4,0.4) (3.2,0.8) (4,1.2) (4.8,1.6) (5.2,1.8)};
\draw [red] plot [smooth] coordinates {(13.2,0.04) (12.8,0.02) (12.4,0.01) (12,0.02) (11.6,0.04) (10.8,0.1) (10,0.3) (9.2,0.72) (8.4,1.4) (8,1.9)};

\node[align=center,rotate=90,anchor=south] at (\xstart0,2.5) {\small $\cost_\Sigma(\what{\theta})$};
\node[align=center,rotate=90,anchor=south] at (\xstart1,2.5) {\small $\cost_\Sigma(\what{\theta})$};

\end{tikzpicture}
    \caption{An illustration of the results of Corollary~\ref{cor:trerr_too_high__thm:main} (describing the regime where memorization is necessary), and Corollary~\ref{cor:trerr_too_low__thm:main} (describing the regime where overfitting is harmful). Here $C>1$ is any constant.}
    \label{fig:extremes_results}
\end{figure}

First, we consider the question of when memorization is necessary: what is the cost of choosing an estimator whose training error is too high?
\begin{corollary}\label{cor:trerr_too_high__thm:main}
    In the setting of Theorem~\ref{thm:main}, suppose also that $\sigma^2 \leq \jfunc_\pi^{-1}$. Then
    \[\textnormal{If $\trerr(\what{\theta}) > \jfunc_\pi\sigma^4$ \ then \ }\cost_\Sigma(\what{\theta})>0.\]
    Moreover, any estimator that fails to memorize sufficiently will incur excess prediction error that is linear in $\trerr(\what{\theta})$: for any $C>1$,
    \[\textnormal{If $\trerr(\what{\theta}) \geq C\cdot \jfunc_\pi\sigma^4$ \ then \ }\cost_\Sigma(\what{\theta})\geq C' \cdot  \trerr(\what{\theta}),\]
    where $C'=\lambda_\Sigma^{-1}(1-C^{-1/2})^2$.
\end{corollary}

Next, we turn to the question of when overfitting is harmful: what is the cost of choosing an estimator whose training error is too low?
\begin{corollary}\label{cor:trerr_too_low__thm:main}
    In the setting of Theorem~\ref{thm:main}, suppose also that $\sigma^2 \geq \mathsf{V}_\pi$. Then
    \[\textnormal{If $\trerr(\what{\theta}) < \sigma^2 - \mathsf{V}_\pi$ \ then \ }\cost_\Sigma(\what{\theta})>0.\]
    Moreover, any estimator that exhibits too much overfitting will incur excess prediction error that is quadratic in $\sigma^2 - \trerr(\what{\theta})$, i.e., in the amount of overfitting: for any $C>1$,
    \[\textnormal{If $\trerr(\what{\theta}) \leq  \sigma^2 - C\cdot \mathsf{V}_\pi$ \ then \ }\cost_\Sigma(\what{\theta})\geq \frac{C'}{\sigma^2}\cdot \prn{\sigma^2 - \trerr(\what{\theta})}^2,\]
    where $C'=(4\lambda_\Sigma)^{-1}(1-C^{-1})^2$.
\end{corollary}

These results are illustrated in Figure~\ref{fig:extremes_results}.

\begin{remark}
    We point out that our characterizations are the cleanest at the extremes, when $\sigma^2 \geq \mathsf{V}_\pi$ or $\sigma^2 \leq \mathsf{J}_\pi^{-1}$. As we will show in Fig.~\ref{fig:monotonic}, under the moderate noise level $\mathsf{J}_\pi^{-1} \leq \sigma^2 \leq \mathsf{V}_\pi$, the behavior of the optimal training error relative to the noise size $\trerr(\what{\theta})/\sigma^2$ can oscillate and be non-monotonic. It remains an interesting open question to investigate for a similar clean characterization for the necessity or harmfulness of overfitting in this regime.
\end{remark}

\subsubsection{The role of the prior: connections to effective dimension}\label{sec:discuss_effective_dim}
In the results above, we have seen that if $\sigma^2$ is close to zero then memorization appears necessary, while if $\sigma^2$ is large then overfitting is harmful. How does this relate to the connection between overfitting and model complexity? In particular, we might expect the following:
\begin{itemize}
    \item If $\theta$ is assumed to lie in a low-complexity model class (e.g., under sparsity), then any accurate estimator $\what{\theta}$ should have minimal overfitting.
    \item On the other hand, if we are in a truly high-dimensional setting (with no assumptions such as sparsity that would induce low-dimensional structure), then we may prefer to overfit, i.e., ``benign overfitting''.
\end{itemize}
In the Bayesian setting considered here, our results can be related to these principles by considering the parameters $\mathsf{V}_\pi$ and $\jfunc_\pi^{-1}$. Essentially, these parameters capture a notion of \emph{effective dimension}.

For intuition, suppose $n=d$ and $X$ is the identity.
First, suppose our prior $\pi$ is quite flat and uninformative. In this case, we would expect $\mathsf{V}_\pi$ and $\jfunc_\pi^{-1}$ to both be $O(1)$ (for instance, if $\pi=\normal(0,I_d)$ then $\mathsf{V}_\pi=\jfunc_\pi^{-1}=1$). In particular, if the noise level $\sigma^2$ is $o(1)$ then memorization becomes necessary: Corollary~\ref{cor:trerr_too_high__thm:main} tells us that cost is high for any estimator with training error $\gtrsim\sigma^4$.

At the other extreme, we can consider a setting where $\pi$ encodes some knowledge of low-dimensional structure in $\theta$---for instance, $\pi$ may be concentrated near some low-dimensional subspace or manifold in $\R^d$. In such a setting, we would expect that overfitting would be harmful, since an optimal estimator should lie near this low-dimensional region and therefore would not have the capacity to overfit. For this type of prior $\pi$, $\jfunc_\pi^{-1}$ would typically be quite low, and the regime in Corollary~\ref{cor:trerr_too_high__thm:main} (i.e., the regime in which our results show that memorization is necessary) becomes negligible. However, there is a subtlety here: if $\sigma^2$ is extremely close to zero (i.e., $\sigma^2\leq \jfunc_\pi^{-1} = o(1)$), then we again find that memorization is necessary. This is an artifact of our Bayesian framework: even if $\pi$ is strongly concentrated around some low-dimensional region within $\R^d$, since we require $\pi$ to have a positive and differetiable density, if the noise level is extremely low then $\pi$ is still effectively constant within any sufficiently small neighborhood. Thus if $\sigma^2$ is extremely small, $\pi$ again acts (locally) as an uninformative prior (see Figure~\ref{fig:illustrate_prior_zoom_in} for an illustration). On the other hand, if we choose $\pi$ to be a distribution that places all its mass on a low-dimensional support (and therefore, $\pi$ cannot have a density), this phenomenon would not occur---but this is beyond the scope of our framework.

In the next section, we explore these ideas through concrete examples to demonstrate the role of effective dimension in this problem.

\begin{figure}[t]
    \centering
    \includegraphics[width=0.55\textwidth]{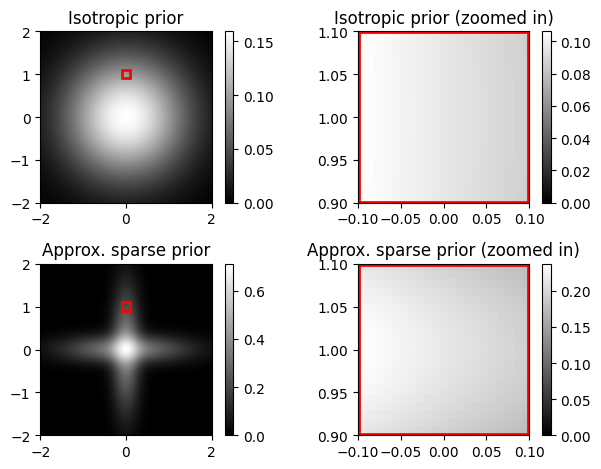}
    \caption{An illustration of the phenomenon discussed in Section~\ref{sec:discuss_effective_dim}. In the top row, we plot an isotropic prior, $\pi = \normal(0,I_2)$, while the bottom row shows a prior that encourages approximate $1$-sparsity, $\pi = 0.5 \normal(0,e_1e_1^\top + \eta e_2e_2^\top) + 0.5 \normal(0,\eta e_1 e_1^\top +  e_2 e_2^\top)$, for $\eta = 0.05$. However, when we zoom in to  the neighborhood of a single point $(0,0.5)$, the two priors are both essentially constant.}
    \label{fig:illustrate_prior_zoom_in}
\end{figure}

\section{Examples with differing levels of complexity} \label{sec:applications}

In this section, we apply our results to investigate concrete setups of $\model_X(\pi, \sigma^2)$ under different priors for the key determining parameters $ \mathsf{V}_{\pi}, \jfunc_{\pi}, \lambda_{\Sigma}$ in Theorem~\ref{thm:main}.

Throughout this section, we will work in a setting where the training data points $x_i$ and the test data point $x'$ all assumed to have mean zero and covariance $\Sigma = I_d$.
To characterize high dimensional behaviors by exact limits, we leverage tools in high dimensional random matrix theory under proportional asymptotics~\citep{bai2010spectral}, working under the following distributional assumptions on the training data features $x_1, \cdots, x_n \in \R^d$ and overparameterized proportional asymptotics. 
\begin{assumption}[Overparameterized proportional asymptotics] \label{assmp:proportional-asymptotics} For a given $\gamma \in (1, +\infty)$, consider a deterministic sequence of positive integers $d(n)$ indexed by $n$, such that $d/n \to \gamma$. For an infinite array of i.i.d.\ random variables $\{X_{ij}\}_{i,j=1}^\infty$ such that $\E[X_{ij}]=0$, $\var(X_{ij})=1$ with bounded fourth-moment $\E[X_{ij}^4] < +\infty$, define $X(n)=(X_{ij})_{i \in [n], j \in [d]} \in \R^{n \times d}$. We consider the sequence of linear models $\model_X(\pi, \sigma^2)=\model_{X(n)}(\pi(n), \sigma^2)$ indexed by $n$. 
\end{assumption}
\subsection{Isotropic Gaussian prior} \label{sec:example_1}
We start with the simplest prior of the isotropic Gaussian $\pi = \normal(0, I_d/d)$. The following proposition shows the implications of our general results for this isotropic setting.\footnote{This is the same setup considered in~\cite{cheng2022memorize}, except that the authors originally examine $X(n)$ whose rows have a general covariance structure $\Sigma(n)$ with bounded condition numbers and converging empirical spectrum distributions (and measure test error against the same covariance $\Sigma(n)$); however, the qualitative conclusions remain the same by taking $\Sigma(n) = I_d$. See Appendix~\ref{proof:applications-isotropic} for extensions of our results to the case of a general $\Sigma(n)$.}

\begin{proposition} \label{prop:applications-isotropic} Let Assumption~\ref{assmp:proportional-asymptotics} hold. For $\sigma^2 > 0$ and $\Sigma = I_d$, we have almost surely
\begin{align*}
    &\mathsf{V}_{\pi} \to 1, \qquad \jfunc_{\pi} \to \frac{\gamma}{\gamma - 1}, \qquad \lambda_{\Sigma} \to (1 + \sqrt{\gamma})^2,
\end{align*}
\end{proposition}

Recalling the results of Section~\ref{sec:main_results_tradeoff} (and Corollary~\ref{cor:trerr_too_high__thm:main} in particular), we see that memorization is necessary when the noise is not too high, since $\mathsf{V}_{\pi}$ and $\jfunc_\pi$ are $O(1)$. We can interpret this finding as showing that, when $\pi$ is essentially uninformative, we are in a genuinely overparameterized regime: in the absence of any implicit structural knowledge of $\theta$, overfitting is unavoidably necessary, and strong predictive performance requires memorization.
\subsection{Approximately-low-rank Gaussian prior} \label{sec:example_2}
In the second example, we begin to investigate scenarios characterized by an underlying underparameterized structure, in which the model lacks sufficient capacity to overfit the noise. To capture this structure while still ensure $X\theta$ has a density $p$, we consider an approximately low-dimensional prior:
\begin{align}\label{eqn:pi_when_eta>0}
    \pi = \pi(n) = \normal(0, \Omega(n)), \qquad \Omega = \Omega(n) = \frac{1-\eta}{r} \begin{bmatrix} I_r & 0 \\ 0 & 0 \end{bmatrix} + \frac{\eta}{d} I_d, \qquad 0 < \eta \leq 1.
\end{align}
For small $\eta$, this prior is approximately supported on a low-rank subspace (spanned by the first $r$ basis vectors). As $\eta \to 0+$, if $r<n$ then the model approaches an underdetermined linear model (where $\trerr(\what{\theta})$ cannot be zero for any estimator), while if $r\geq n$ then the model essentially approaches the same setting as the isotropic Gaussian example of Section~\ref{sec:example_1}, except with $r$ in place of $d$.
\begin{proposition} \label{prop:applications-perturbed-low-rank-positive-eta} Let Assumption~\ref{assmp:proportional-asymptotics} hold, $r/n \to \rho \in (0, +\infty)$, and $\pi$ defined as in~\eqref{eqn:pi_when_eta>0} above. Define the constant $C=C(\gamma, \rho ,\eta) =\eta^2(\gamma-\rho-1)+\eta(\gamma \rho-\rho^2-\rho)+\gamma(\rho-1)$. For $\sigma^2 > 0$ and $\Sigma = I_d$,
\begin{align*}
    &\mathsf{V}_{\pi} \to 1, \qquad \jfunc_{\pi} \to \frac{-C+\sqrt{C^2+4 \eta(\gamma-1) \gamma(\eta+\rho)}}{2 \eta(\gamma-1)}, \qquad \lambda_{\Sigma} \to (1 + \sqrt{\gamma})^2
\end{align*}
almost surely.
In particular, we have the following limits as $\eta \to 0+$:
\begin{align*}
    \text{Underparameterization limit: if $\rho < 1$,} & \qquad \qquad \jfunc_{\pi} = \frac{\gamma(1-\rho)}{\eta(\gamma - 1)} + o(\eta^{-1}); \\
    \text{Interpolation threshold limit: if $\rho = 1$,} & \qquad \qquad \jfunc_{\pi} = \sqrt{\frac{\gamma}{\eta(\gamma-1)}} + o(\eta^{-\half});\\
    \text{Overparameterization limit: if $\rho > 1$,} & \qquad \qquad \jfunc_{\pi} = \frac{\rho}{\rho-1} + o(1).
\end{align*}
\end{proposition}

In the regime $\rho>1$ (i.e., $r>n$), even as $\eta\to 0+$ the model is still overparameterized, and essentially reduces to the same result as the isotropic Gaussian, in Proposition~\ref{prop:applications-isotropic}: we have $\jfunc_\pi\to\frac{\rho}{\rho-1}$ in place of $\frac{\gamma}{\gamma-1}$, since the effective dimension is $r=\rho n$ instead of $d = \gamma n$. In particular, memorization is again necessary, since the results of Corollary~\ref{cor:trerr_too_high__thm:main} apply as soon as $\sigma^2\leq\jfunc_\pi^{-1} \asymp 1$. 

On the other hand, if $\rho<1$ (i.e., $r<n$), we see different behavior. Indeed, if $\eta$ is close to zero (i.e., the prior is strongly concentrated near the low-rank subspace), then if the noise $\sigma^2$ is not too small, intuitively we would expect overfitting to be harmful since we have knowledge of low-rank structure in $\theta$ (as expressed by $\pi$).
This is confirmed by the theory: by Proposition~\ref{prop:applications-perturbed-low-rank-positive-eta} and Corollary~\ref{cor:trerr_too_high__thm:main}, as $\eta\to0+$ we see that memorization is necessary once $\sigma^2\leq \jfunc_\pi^{-1} = O(\eta)$, but if $\eta\to0+$ then any positive noise level will not fall into this regime. 

For comparison, we can consider a prior with exact low-rank structure,
\begin{equation}\label{eqn:pi_when_eta=0}(\theta_1,\dots,\theta_r)^\top \sim \normal(0,I_r/r), \qquad (\theta_{r+1},\dots,\theta_d)=0,\end{equation}
which we can view as a limit of~\eqref{eqn:pi_when_eta>0} by setting $\eta=0$.
This prior does not have a density on $\R^d$ and thus lies outside the framework of our main results in Theorem~\ref{thm:Bayes_train_err}, which bounds the training error of the Bayes estimator---but in this simple setting, we can instead compute $\trerr(\bysest)$ directly.

\begin{proposition} \label{prop:applications-perturbed-low-rank-zero-eta} Let Assumption~\ref{assmp:proportional-asymptotics} hold, let $r/n \to \rho \in (0, 1)$, and let $\sigma^2>0$ be fixed. Then, for the low-rank prior $\pi$ defined in~\eqref{eqn:pi_when_eta=0}, it holds almost surely that
$\trerr(\bysest) \to \sigma^2 \cdot (1-\rho) +\sigma^4 \cdot \frac{\rho^2}{1-\rho}  + O(\sigma^6)$.
\end{proposition}
Consequently, any near-optimal estimator $\what{\theta}$ must have, at most, a bounded amount of overfitting: by Proposition~\ref{prop:compare_to_Bayes}, we will have $\cost_\Sigma(\what{\theta})\gtrsim \sigma^2$ if $\sigma^2 - \trerr(\what{\theta}) \geq c\sigma^2$ for any constant $c>\rho$.

\subsection{Mixture of approximately-sparse priors}\label{sec:example_3}

We now move beyond the Gaussian prior setting, in which the absence of a closed-form expression for $\nabla \log p(y)$ makes it challenging for exact evaluation of the key Fisher information quantity $\jfunc_\pi$. Nevertheless, in certain cases we can still derive bounds that characterize its behavior. In particular, we consider the mixture of the collection of perturbed $K$-sparse priors:
\begin{align}\label{eqn:pi_approx_sparse}
    \pi = \pi(n) = \frac{1}{\binom{d}{K}} \sum_{S \subset [d], |S|=K} \normal(0, \Omega_S(n)), \qquad \Omega_S = \Omega_S(n) = \frac{1-\eta}{K}\sum_{i \in S} e_ie_i^\top + \frac{\eta}{d} I_d.
\end{align}
This prior is appropriate for a setting where we believe $\theta$ is $K$-sparse, but have no knowledge of its support. The distribution of $X\theta$ induced by $\theta \sim \pi$ then has the density
\begin{align*}
    p(y) = \frac{1}{\binom{d}{K}} \sum_{S \subset [d], |S|=K} \phi(y; 0, X \Omega_S X^\top),
\end{align*}
where $\phi(\cdot \, ; \, \mu, \Sigma)$ denotes the density of the $\normal(\mu,\Sigma)$ distribution.

\begin{proposition} \label{prop:applications-perturbed-sparse} Let Assumption~\ref{assmp:proportional-asymptotics} hold, and $\pi$ defined as in~\eqref{eqn:pi_approx_sparse} above. Assume also that $K/n \to 0$. For $\sigma^2 > 0$ and $\Sigma = I_d$, we have almost surely
\[\frac{1}{\eta} \cdot e \prn{1-\frac{1}{\gamma}}^{\gamma - 1} \leq \lim\inf_{n\to\infty} \jfunc_\pi \leq \lim\sup_{n\to\infty} \jfunc_\pi \leq \frac{\gamma}{\eta(\gamma-1)}.\]
In addition, when $\gamma \to +\infty$, both sides converge to $1/\eta$.
\end{proposition}

Assuming $K/n\to 0$ ensures that, under the prior $\pi$, the model is essentially low-dimensional. Similarly to the results of Section~\ref{sec:example_2} (for the underparameterized case $r<n$ considered there), here we see that, provided $\eta\approx 0$---i.e., the prior is strongly concentrated on approximately-sparse values of $\theta$---the results only suggest that memorization is necessary if the noise level $\sigma^2$ is extremely small.

\section{Examining the training error of the Bayes estimator} \label{sec:fine-grained}
As implied by Theorem~\ref{thm:main}, quantifying the cost of overfitting (or not overfitting) hinges on characterization of the training error achieved by the optimal Bayes estimator under a given $\model_X$ specification. In this regard, Theorem~\ref{thm:Bayes_train_err} constitutes the principal driving mechanism for the central results presented in this work. 
In this section we examine this training error more closely:
\begin{enumerate}
    \item In Sec.~\ref{sec:Tweedie}, we present Tweedie's formula and a Fisher information expression for $\trerr(\bysest)$ valid for all $\sigma^2 > 0$, which serve as the foundation in the proof of Theorem~\ref{thm:Bayes_train_err} in Appendix~\ref{proof:Bayes_train_err}.
    \item In Sec.~\ref{sec:monotonicity}, we analyze monotonicity properties of $\trerr(\bysest)$ w.r.t.\ the noise level $\sigma^2$.
\end{enumerate} 

To simplify notation, throughout this section, we denote the density of $\normal(0, \sigma^2 I_n)$ by $\phi_{\sigma^2}$. Let $\pi'$ be the probability distribution of $X\theta$ in $\R^n$, and let $\pi'_{\sigma^2} := \pi'*\normal(0, \sigma^2 I_n)$, which is the marginal distribution of $y$. We also assume that the density $p > 0$ of $\pi'$ always exists and is differentiable, and continue to assume that $\pi$ has mean zero, as before. Write $p_{\sigma^2} = p\ast \phi_{\sigma^2}$ as the density of $\pi'_{\sigma^2}$.\footnote{The differentiability assumption can be relaxed in certain cases, which we will discuss in the Appendix.}  Then
\begin{align*}
    \trerr\prn{\bysest} & = \frac{1}{n}\E_{y \sim \pi'_{\sigma^2}} \brk{\norm{X \bysest - y}^2 \mid X} = \frac{1}{n}\E_{\theta \sim \pi, \tau \sim \normal(0, I_n)} \brk{\ltwo{X \bysest(X, X\theta + \sigma \tau) - X \theta - \sigma \tau}^2 \mid X}.
\end{align*}
The above manner in which $\trerr(\bysest)$ depends on the specifications of the linear model $\model_X$ is not transparent. Thanks to the remarkable representation of the posterior mean through the score function, known as Tweedie’s formula~\citep{efron2011tweedie} in Bayesian estimation, we can concisely express $\trerr(\bysest)$ in terms of the Fisher information of $\pi_{\sigma^2}'$.
\subsection{Tweedie's formula and Fisher information} \label{sec:Tweedie}
By Tweedie's formula,
$$
\E [X\theta \mid X, y] = y + \sigma^2 \nabla \log p_{\sigma^2}(y),
$$
and from this we can immediately see how to characterize the training error by the Fisher information of $\pi'_{\sigma^2}$ by rearranging terms and taking expectation over $y$. 
\begin{lemma} \label{lem:bayes-train-err-representation}
It holds for all noise level $\sigma^2 > 0$ that
\begin{subequations}
\begin{align} \label{eq:bayes-train-err-cgf}
\trerr\prn{\bysest} = \frac{\sigma^2}{n} \cdot \E_{\theta \sim \pi, \tau \sim \normal(0, I_n)} \brk{\norm{\frac{\partial}{\partial \tau} \log \E_{\theta' \sim \pi} \brk{\exp \brc{-\frac{\norm{X \theta' - X\theta - \sigma \tau}^2}{2 \sigma^2}} \mid X, y}}^2 \mid X}.
\end{align}
Moreover, defining the Fisher information matrix of $\pi'_{\sigma^2}$ as
\[
\ifunc(\sigma^2) := \E_{y \sim \pi'_{\sigma^2}} \brk{ \nabla\log  p_{\sigma^2}(y) \nabla \log  p_{\sigma^2}^\top(y)} 
= - \E_{y \sim \pi'_{\sigma^2}} \brk{ \nabla^2 \log  p_{\sigma^2}(y)},
\]
it holds that
\begin{align}  \label{eq:bayes-train-err-fisher}
\trerr\prn{\bysest} = \frac{\sigma^4}{n} \cdot  \E_{y \sim \pi'_{\sigma^2}} \brk{ \norm{\nabla \log  p_{\sigma^2}(y)}^2} = \frac{\sigma^4}{n} \cdot  \Tr \prn{\ifunc(\sigma^2)}.
\end{align}
\end{subequations}
\end{lemma}

\begin{remark} \label{remark:suggested-asymptotics} Lemma~\ref{lem:bayes-train-err-representation} suggests that the trace of the Fisher information matrix $\ifunc(\sigma^2)$ governs the training error of the Bayes estimator $\bysest$. In the low noise regime $\sigma^2 \to 0^+$, we anticipate
\[
\trerr \prn{\bysest} = (1+o(1)) \frac{\sigma^4}{n} \lim_{\sigma^2 \to 0^+} \Tr \prn{\ifunc(\sigma^2)}.
\]
Higher-order expansions of $\trerr \prn{\bysest}$ can be obtained from the asymptotics of $\Tr \prn{\ifunc(\sigma^2)}$ as $\sigma^2 \to 0+$, which we defer to Appendix~\ref{sec:higher-order}. As $\sigma^2 \to +\infty$, the distribution $\pi'_{\sigma^2}$ converges to $\normal(0,\sigma^2 I_n)$, and the Fisher information satisfies $\ifunc(\sigma^2)= (1+o(1))I_n/\sigma^2$. Hence, as $\sigma^2\to+\infty$,
\[
\trerr \prn{\bysest}
= (1+o(1))\frac{\sigma^4}{n} \cdot \frac{n}{\sigma^2}
= (1+o(1))\sigma^2.
\]
\end{remark}

\subsection{Monotonicity of the training error} \label{sec:monotonicity}

Overloading notation, we next define $\tfunc(\sigma^2)$ as a function on $(0, +\infty)$ to represent the training error of $\trerr(\bysest)$ at noise level $\sigma^2$---that is, $\tfunc(\sigma^2)$ is defined as the right-hand side of~\eqref{eq:bayes-train-err-cgf}. We also define an auxiliary function $\jfunc$ on $(0, +\infty)$ for the Fisher information as
\begin{align}
\jfunc(\sigma^2) & =  \frac{1}{n} \cdot  \E_{y \sim \pi'_{\sigma^2}} \brk{ \norm{\nabla \log  p_{\sigma^2}(y)}^2}  = \frac{1}{n} \Tr (\ifunc(\sigma^2)). \label{eq:jfunc}
\end{align}
By the equivalence of Eqs.~\eqref{eq:bayes-train-err-cgf} and \eqref{eq:bayes-train-err-fisher} from Lemma~\ref{lem:bayes-train-err-representation}, we have $\tfunc(\sigma^2) = \sigma^4 \jfunc(\sigma^2)$. We also extend the definition to $\jfunc(0) = \jfunc_{\pi}$ and $\tfunc(0) = 0$ if $\pi$'s push-forward distribution $\pi'$ has finite Fisher information. 

A natural direction for deepening our understanding of the behavior of the Bayes estimator is to investigate the evolution---in particular, monotonicity---of $\tfunc(\sigma^2)$ as a function of the noise variance parameter $\sigma^2$. Intuitively, $\tfunc(\sigma^2)$ should increase as the additional noise makes the optimal learning procedure less confident. Employing powerful tools for Gaussian channels~\citep{guo2004mutual, guo2011estimation} enables us to conclude the following.
\begin{proposition}
\label{prop:monotonicity} 
    If $\jfunc_\pi$ is finite, then $\tfunc(\sigma^2)$ is monotonically non-decreasing on $[0, +\infty)$, and $\tfunc(\sigma^2)/\sigma^4 = \jfunc(\sigma^2)$ is monotonically non-increasing on $[0, +\infty)$.
\end{proposition}
However, the function $\tfunc(\sigma^2)/\sigma^2$ (i.e., the training error relative to the noise level) can in general be non-monotonic: see Fig.~\ref{fig:monotonic} for an example\footnote{See the code for the experiment in \hyperlink{blue}{https://github.com/Moriartycc/is-memorization-helpful-or-harmful}}.
\begin{figure} 
    \centering
    \includegraphics[width=.95\linewidth]{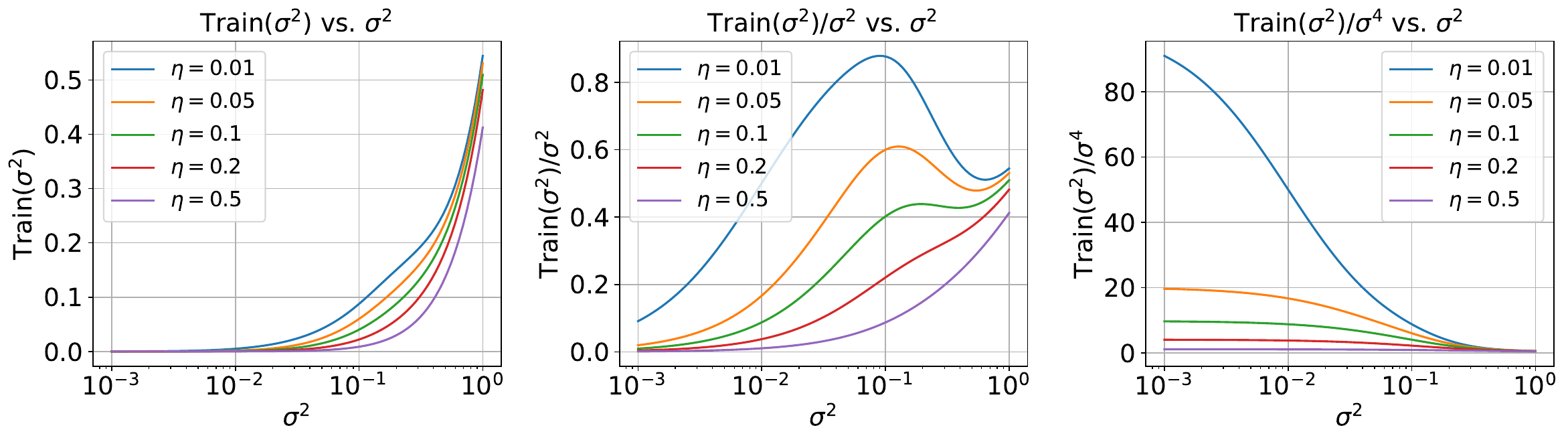}
    \caption{Numerical simulations for $\pi' = 0.5 \normal(-1, \eta) + 0.5 \normal(1, \eta)$. 
    From left to right, we plot $\tfunc(\sigma^2)$, $\tfunc(\sigma^2)/\sigma^2$ and $\tfunc(\sigma^2)/\sigma^4$ vs. $\sigma^2$.} \label{fig:monotonic}
\end{figure}

\section{Discussion} \label{sec:discussion}
Our results, extending the memorization phenomenon in~\citet{cheng2022memorize} from Gaussian priors to general priors, provide summarizing parameters $\jfunc_\pi^{-1}$ and $\mathsf{V}_\pi$, which are intrinsic to the underlying parameter distribution $\pi$ and capture a notion of effective dimension. These parameters yield asymptotically optimal guaranties for characterizing when interpolation is necessary or harmful. Furthermore, Theorem~\ref{thm:Bayes_train_err} does not rely on proportional asymptotics nor assume a high-dimensional ambient space, and thus holds in a fully general setting.

For linear models, the hypothesis class consists essentially of the parametric family  $\mathcal{F}_\Theta = \{x^\top \theta \mid \theta \in \Theta = \mathbb{R}^d\}$. It is of considerable interest to extend the present analysis to more general model classes \(\mathcal{F}_\Lambda = \{f_\lambda\}\), where \(\Lambda\) may be either a parametric or a nonparametric index set; it would also be valuable to investigate analogous phenomena outside the Bayesian framework without imposing a prior distribution on \(\Theta\).  

Finally, the role of \(\lambda_\Sigma\) is inherently tied to the finite-dimensional ambient space \(d < \infty\). In Hilbert spaces~\citep{bartlett2020benign, cheng2024dimension}, one has \(\lambda_\Sigma = \infty\), while the cited works show that benign overfitting occurs for arbitrary noise levels in such settings. A future theory of memorization in Hilbert spaces will therefore require a fundamentally different analytical approach.

\section*{Acknowledgment}
C.C. and R.F.B. were supported by the Office of Naval Research via grant N00014-24-1-2544. R.F.B. was additionally supported by the National Science Foundation via grant DMS-2023109.

\bibliography{bib}
\bibliographystyle{abbrvnat}

\newpage

\appendix

\section{Technical lemmas} \label{proof:technical}

In this section, we gather some technical lemmas from existing work, that will be helpful in the main proofs.

The first result concerns the computation of moments of the multivariate normal distribution. \citeauthor{isserlis1918formula} originally established the general formula for arbitrary moments, whereas we require only the following identity for the fourth-order moment.
\begin{lemma}[Isserlis's theorem~\citep{isserlis1918formula}] \label{lem:isserlis}
    Let $X=(X_1,X_2,X_3,X_4)$ be a zero-mean Gaussian random vector. The following identity holds:
    \begin{align*}
        \E[X_1X_2X_3X_4] = \E[X_1 X_2] \E[X_3 X_4]+\E[X_1 X_3] \E[X_2 X_4]+\E[X_1 X_4] \E[X_2 X_3] .
    \end{align*}
\end{lemma}

The next result is \citeauthor{stam1959some}'s Fisher information inequality~\citep{stam1959some} for sums of independent random variables, based on \citeauthor{shannon1948mathematical}'s entropy power inequality~\citep{shannon1948mathematical}. While \citeauthor{stam1959some}'s original result is stated for scalar random variables, a vector-valued extension can be found in~\cite[Thm.~13]{dembo2002information}. We state an equivalent version given 
in \cite[Eq.~(37)]{dembo2002information} as follows.

\begin{lemma}[Fisher information inequality] \label{lem:FII} For two independent vectors $X, Y \in \R^n$ with densities $p, q$, let $Z = X + Y$ with density $p*q$. The following inequality holds whenever all quantities are well-defined:
\begin{align*}
    \frac{1}{\E_{Z} \brk{\norm{\nabla \log (p*q)(Z)}^2}} \geq \frac{1}{\E_{X} \brk{\norm{\nabla \log p(X)}^2}} + \frac{1}{\E_{Y} \brk{\norm{\nabla \log q(Y)}^2}}.
\end{align*}
\end{lemma}
In the same paper~\citep{stam1959some}, \citeauthor{stam1959some} provides a lower bound for the Fisher information by the entropy power of a random variable, sharpening the Cram\'{e}r-Rao bound by the variance. The result, often referred to as entropic isoperimetric inequality, is then later extended into the vector form. Here we present the version of the result stated in~\cite[Thm.~16]{dembo2002information}.
\begin{lemma}[Entropic isoperimetric inequality] \label{lem:EII} For a random vector $X \in \R^n$ with density $p$ and for which the Fisher information exists,
\begin{align*}
    \frac{1}{n} \E_{X} \brk{\norm{\nabla \log p(X)}^2} \geq 2\pi e \cdot \exp \brc{-\frac{2}{n} \E_{X} \brk{-\log p(X)}}.
\end{align*}   
The equality holds iff $X \sim \normal(0, \nu^2 I_n)$ for some $\nu^2 > 0$.
\end{lemma}

The following lemmas cover several classical results concerning high-dimensional sample covariance matrices within the framework of random matrix theory that will be useful in our analyses for specific examples in Sec.~\ref{sec:applications}. These include (i) the Bai–Yin law, which characterizes the asymptotic behavior of the largest and smallest eigenvalues, and (ii) the Marchenko–Pastur law, which describes the limiting empirical spectral distribution. We describe our results in a more general setting than Assumption~\ref{assmp:proportional-asymptotics}, relaxed to also incorporate the underparameterized regime $\gamma \in (0, +\infty)$. We also assume the data features $X=Z\Sigma^{\half}$ such that $Z$ has i.i.d.\ entries with zero mean and unit variance and bounded fourth-moment, and we let $\Sigma= \Sigma(n)$ form a sequence of p.s.d.\ matrices. Let $\what{S}=XX^\top / d \in \R^{n \times n}$ and denote by the nonzero eigenvalues of $\what{S}$ by $\lambda_1 \geq \cdots \lambda_{n \vee d} > 0 = \lambda_{n \vee d + 1} = \cdots = \lambda_n$ and the empirical spectrum distribution (e.s.d.)\ by $\mu_n(\de \lambda) = \frac{1}{n} \sum_{i=1}^n \delta_{\lambda_i}$, we have the standard Bai-Yin and Marchenko-Pastur theorems for $\Sigma = I$ (cf.~\cite[Thm.~3.6 and Thm.~5.10]{bai2010spectral}).
\begin{lemma}[Marchenko-Pastur and Bai-Yin laws] \label{lem:MP-BY-law}
    Let $d/n \to \gamma \in (0, +\infty)$. The largest and smallest nonzero eigenvalues of $\what{S}$ satisfies the Bai-Yin theorem:
    \begin{align*}
        \lambda_1 \cas  \prn{1 + \gamma^{-\half}}^2  =: \lambda^+_\gamma, \qquad \lambda_n \cas  \prn{1 - \gamma^{-\half}}^2 =:\lambda^-_\gamma.
    \end{align*}
    With probability one, the e.s.d.\ $\mu_n$ of $\what{S}$ converges weakly to the Marchenko-Pastur law:
    \begin{align*}
        \mu_n(\de \lambda) \cd \mu_{\mathsf{MP}, \gamma}(\de \lambda) = (1 - \gamma)_+ \delta_0 + \frac{\gamma}{2\pi} \cdot \frac{\sqrt{(\lambda^+_\gamma-\lambda)(\lambda - \lambda^-_\gamma)}}{\lambda} \ind_{\lambda^-_\gamma \leq \lambda \leq \lambda^+_\gamma} \de \lambda 
    \end{align*}
\end{lemma}
We then consider the case when $\Sigma$ is a deterministic sequence of p.s.d.\ matrices whose e.s.d.\ converges. In particular, let $\sigma_1 \geq \cdots \geq \sigma_d \geq 0$ denote the eigenvalues of $\Sigma$ and $\nu_n(\de \lambda) = \frac{1}{d} \sum_{i=1}^d \delta_{\sigma_i}$. To facilitate the statement, we introduce the Stieltjes transform $m_{\mu}(z): \mathbb{C}_+ \to \mathbb{C}_+$ of a probability measure $\mu$ on $\R$, where $\mathbb{C}_+=\{u+iv \mid v>0\}$. $m_{\mu}(z)$ is then
\begin{align} \label{eq:Stieltjes-transform}
    m_{\mu}(z) = \int \frac{1}{\lambda - z} \mu(\de \lambda).
\end{align}
It is well-defined everywhere on $\C_+$ since $|(\lambda-z)^{-1}| = |(\lambda - u) + iv|/((\lambda-u)^2 + v^2) \leq 1/|\Im z|$. The Stieltjes transform uniquely determines any finite signed measure on $\R$, see the inversion formula in~\cite[Thm.~B.8]{bai2010spectral} to recover $\mu$ for $m_{\mu}(z)$. The following generalized version of Lemma~\ref{lem:MP-BY-law} then holds (cf.~\cite[Thm.~4.3 and Thm.~6.3]{bai2010spectral}).
\begin{lemma}[Generalized Marchenko-Pastur and Bai-Yin laws] \label{lem:generalized-MP-BY-law}
    Let $d/n \to \gamma \in (0, +\infty)$, $\Sigma_n$ have bounded spectral norm as $n\to+\infty$, and assume $\nu_n \cd \nu$ with bounded support. With probability one, the e.s.d.\ $\mu_n$ of $\what{S}$ converges weakly to $\mu$ determined by $m_{\mu}(z)$ from the following fixed-point equation on $\C_+$:
    \begin{align*}
        m_\mu(z) = - \prn{z-\int \frac{\lambda}{1 + \lambda m_\mu(z) \gamma^{-1}} \nu(\de \lambda)}^{-1}.
    \end{align*}
    In addition when $\gamma > 1$, if the smallest and largest eigenvalues of $\Sigma_n$ converge to the smallest and largest numbers in the support of $\nu$, then $\lambda_1$ and $\lambda_n$ converge almost surely to the smallest and largest numbers in the support of $\mu$.
\end{lemma}
We can also efficiently compute the limit of $\lambda_1$ and $\lambda_n$ by the following lemma~\cite[Thm.~4.2]{silverstein1995analysis}.
\begin{lemma} \label{lem:edge-equations}
    Under the same setup of Lemma~\ref{lem:generalized-MP-BY-law} for $\gamma > 1$. Let $\lambda_{\mu}^+$ and $\lambda_{\mu}^-$ be the largest and smallest numbers to satisfy the following equation
    \begin{align*}
        \frac{1}{m_{\mu}(\lambda)^2} = \frac{1}{\gamma} \int \frac{\lambda^2}{(1+(\lambda / \gamma) m_{\mu}(\lambda))^2} \nu(d \lambda),
    \end{align*}
    in addition to the fixed-point equation in Lemma~\ref{lem:generalized-MP-BY-law} and $m_{\mu}(\lambda) < 0$. Then $\lambda_1 \to \lambda_\mu^+$ and $\lambda_n \to \lambda_\mu^-$ with probability one.
\end{lemma}

\section{Proofs for Section~\ref{sec:main}}

\subsection{Proof of Proposition~\ref{prop:compare_to_Bayes}}
    For any estimator $\what{\theta}$,
\begin{align}
    \prerr_\Sigma(\what{\theta}) & = \E \brk{ \norm{\what{\theta} - \theta}_\Sigma^2 \mid X} = \E \brk{ \norm{\what{\theta} - \bysest+ \bysest-\theta}_\Sigma^2 \mid X} \nonumber \\
    & = \E \brk{ \norm{\what{\theta} - \bysest}_{\Sigma}^2 + \norm{\bysest -\theta}_\Sigma^2 \mid X} = \prerr_\Sigma(\bysest) + \E \brk{\norm{\what{\theta} - \bysest}_{\Sigma}^2\mid X}, \nonumber 
\end{align}
where the next-to-last step holds since $\bysest=\E\brk{\theta\mid X,y}$ and so
\[\E\brk{(\what{\theta}-\bysest)^\top\Sigma(\bysest-\theta)\mid X,y} = (\what{\theta}-\bysest)^\top\Sigma\E\brk{\bysest-\theta\mid X,y}=0.\]This verifies the equality $\prerr_\Sigma(\what{\theta}) - \prerr_\Sigma^* = \E \brk{\norm{\what{\theta} - \bysest}_{\Sigma}^2\mid X}$. On the other hand, applying the Minkowski inequality in $\mc{L}^2$ yields
\begin{multline*}
    \left|\sqrt{\trerr(\what{\theta})} - \sqrt{\trerr(\bysest)} \right| = \left|\sqrt{\frac{1}{n} \E \brk{\norm{X \what{\theta} - y}^2 \mid X}} -  \sqrt{\frac{1}{n} \E \brk{\norm{X \bysest - y}^2 \mid X}}\right| \nonumber \\
     \leq \sqrt{\frac{1}{n} \E \brk{\norm{X (\what{\theta} - \bysest)}^2\mid X}}
    \leq \norm{\Sigma^{-\half} \cdot \frac{X^\top X}{n} \cdot \Sigma^{-\half}}^{\half} \cdot \sqrt{\E \brk{\norm{\theta - \bysest}_\Sigma^2\mid X}}.
\end{multline*}

\subsection{Proof of Theorem~\ref{thm:Bayes_train_err}} \label{proof:Bayes_train_err}
We divide our proofs into three parts, addressing the general bound in Eq.~\eqref{eqn:sandwich_ineq_thm:Bayes_train_err}, the low noise asymptotic expansion in Eq.~\eqref{eqn:asymp_low_noise_thm:Bayes_train_err} and the high noise expansion in Eq.~\eqref{eqn:asymp_high_noise_thm:Bayes_train_err} respectively. The proofs, especially the first two parts, rely heavily on several remarkable results from the theory of information inequalities. We first refer the reader to the necessary backgrounds in Sec.~\ref{sec:Tweedie} in which we define the important auxiliary functions $\jfunc(\sigma^2)$ and $\tfunc(\sigma^2)$ such that $\trerr(\bysest) = \jfunc(\sigma^2) = \sigma^4 \tfunc(\sigma^2)$ and importantly $\jfunc(0) = \jfunc_\pi$, whose continuity $\jfunc(0) = \jfunc(0+)$ will be a main focus of our proof. In this proof, we will work under the notation that $t=\sigma^2$. We shall also use $\pi'$ to denote the probability distribution of $X\theta$ on $\R^n$ induced by $\pi$.

\paragraph{Part I: The general lower and upper bounds.} 
The Fisher information upper bound is a direct corollary of Stam's Fisher information inequality in Lemma~\ref{lem:FII}, applying to $\pi'_t = \pi' * \normal(0, t I_n)$ and implying $\jfunc(t)^{-1} \geq \jfunc(0)^{-1} +t$. Therefore,
\begin{align*}
    \tfunc(t) = t^2 \jfunc(t) \leq \frac{t^2}{\jfunc(0)^{-1} +t} = \frac{t^2}{\jfunc_{\pi}^{-1} +t}.
\end{align*}
The lower bound follows from the classical scalar version of Cramer-Rao bound~\cite[Chp.~2]{lehmann1998theory}. We provide a simple two-line proof. Since by Cauchy-Schwarz and integration by parts,
\begin{align*}
    (\mathsf{V}_{\pi'} + t) \jfunc(t) & = \prn{\frac{1}{n}\E_{y \sim \pi_t'} \brk{\norm{y}^2}} \cdot \prn{\frac{1}{n} \E_{y \sim \pi_t'} \brk{\norm{\nabla \log p_t(y)}^2}} \geq \prn{\frac{1}{n}\int y \cdot \nabla \log p_t(y) p_t(y) \de y }^2 \nonumber \\
    & = \prn{-\frac{1}{n}\int \underbrace{\nabla \cdot y}_{=n} p_t(y) \de y }^2 = \prn{\int  p_t(y) \de y }^2 = 1.
\end{align*}
Thus we establish the lower bound:
\begin{align*}
    \tfunc(t) = t^2 \jfunc(t) \geq \frac{t^2}{\mathsf{V}_{\pi} + t}.
\end{align*}

\paragraph{Part II: Asymptotic expansion when $\sigma^2 \to 0+$.} Since $\tfunc(t) = t^2 \jfunc(t)$, we only need to show the continuity of $\jfunc(t)$ at $0+$, i.e.\
\begin{align*}
    \lim_{t \to 0+} \jfunc(t) = \lim_{\sigma^2 \to 0+} \frac{1}{n} \cdot  \E_{y \sim \pi'_{t}} \brk{ \norm{\nabla \log  p_{t}(y)}^2} = \frac{1}{n} \cdot  \E_{y \sim \pi'} \brk{ \norm{\nabla \log  p(y)}^2} = \jfunc(0).
\end{align*}
The right continuity of $\jfunc$ at $0$, seemingly trivial at first glance, indeed requires deep technical tools from information theory. Since $p_t = p * \phi_t$, we first use \citeauthor{stam1959some}'s Fisher information inequality in Lemma~\ref{lem:FII}, which implies $\jfunc(t)^{-1} \geq \jfunc(0)^{-1} + t$ and
\begin{align} \label{eq:mid-low-noise-asymptotics-1}
    \limsup_{t \to 0+} \jfunc(t) \leq \jfunc(0).
\end{align}
It boils down to proving $\liminf_{t \to 0+} \jfunc(t) \geq \jfunc(0)$. Since $\pi_{t}'$ converges weakly to $\pi'$ trivially from $X\theta + \sqrt{t} \tau \cd X \theta$, the desired inequality holds by the lower semi-continuity of Fisher information under weak convergence. For scalar random variables, this is known in~\cite[Prop.~3.1]{bobkov2014fisher} and~\cite[P.~78]{doi:https://doi.org/10.1002/9780470434697.ch4} by the convexity of Fisher information under vague topology. This is indeed true in general dimensions, and we provide a proof in what follows for completeness, using the similar variational principle in~\cite{doi:https://doi.org/10.1002/9780470434697.ch4} and \cite[Thm.~4]{artstein2004solution}. Let $p_0 = p$, we have the following lemma.
\begin{lemma}[Variational principle of the Fisher information] \label{lem:variational-principle} Let $\mc{C}^\infty_c(\R^n; \R^n)$ be the family of all compactly supported and smooth vector fields from $\R^n$ to $\R^n$, then for all $t \geq 0$,
\begin{align*}
    \jfunc(t) = \sup _{b \in \mc{C}_c^{\infty}(\R^n; \R^n)}\left\{- \frac{2}{n} \int p_t(y) \nabla \cdot b(y) \de y- \frac{1}{n}\int p_t(y) \norm{b(y)}^2 \de y\right\} .
\end{align*}
    
\end{lemma}
\begin{proof}
    By nonnegativity
    \begin{align*}
        \frac{1}{n}\int p(y)\norm{b(y) - \nabla \log p_t(y)}^2 \de y = \epsilon_b \geq 0,
    \end{align*}
    we immediately have $\jfunc(t) \geq \epsilon_b + \frac{2}{n} \int p_t(y) b(y) \cdot \nabla \log p_t(y) \de y - \frac{1}{n}\int p_t(y) \norm{b(y)}^2 \de y$. While $b$ is smooth and compactly supported, we can integrate by parts and obtain
    \begin{align*}
        \frac{2}{n} \int p_t(y) b(y) \cdot \nabla \log p_t(y) \de y = \frac{2}{n}  \int b(y) \cdot \nabla p_t(y) \de y = -  \frac{2}{n}  \int p_t(y) \nabla \cdot b(y) \de y.
    \end{align*}
    Here $\nabla \cdot b(y) := \sum_{i=1}^n \frac{\partial}{\partial y_i} b(y)$ denotes the divergence of $b(y)$. Thus taking supremum over all $b \in \mc{C}_c^{\infty}(\R^n; \R^n)$,
    \begin{align*}
        \jfunc(t) & \geq \sup _{b \in \mc{C}_c^{\infty}(\R^n; \R^n)}\left\{ \epsilon_b - \frac{2}{n} \int p_t(y) \nabla \cdot b(y) \de y- \frac{1}{n}\int p_t(y) \norm{b(y)}^2 \de y \right\} \\
        & \geq \sup _{b \in \mc{C}_c^{\infty}(\R^n; \R^n)}\left\{- \frac{2}{n} \int p_t(y) \nabla \cdot b(y) \de y- \frac{1}{n}\int p_t(y) \norm{b(y)}^2 \de y\right\}.
    \end{align*}
    The equality holds since $\epsilon_b$ can be arbitrarily small.
\end{proof}

We are now ready to show lower semi-continuity using Lemma~\ref{lem:variational-principle}. Indeed, since $p_t(y) \de y \cd p(y) \de y$, by weak convergence, for any compactly supported smooth $b$,
\begin{align*}
    &\lim_{t \to 0+} \prn{- \frac{2}{n} \int p_t(y) \nabla \cdot b(y) \de y- \frac{1}{n}\int p_t(y) \norm{b(y)}^2 \de y} \nonumber \\
    &= - \frac{2}{n} \int p(y) \nabla \cdot b(y) \de y- \frac{1}{n}\int p(y) \norm{b(y)}^2 \de y.
\end{align*}
Thus
\begin{align*}
    \liminf_{t \to 0+} \jfunc(t) \geq - \frac{2}{n} \int p(y) \nabla \cdot b(y) \de y- \frac{1}{n}\int p(y) \norm{b(y)}^2 \de y.
\end{align*}
Taking supremum on both sides for $b \in \mc{C}_c^{\infty}(\R^n; \R^n)$ yields
\begin{align} \label{eq:mid-low-noise-asymptotics-2}
    \liminf_{t \to 0+} \jfunc(t) \geq \jfunc(0).
\end{align}
We show $\lim_{t \to 0+} \jfunc(t) = \jfunc(0)$ combining Eqs.~\eqref{eq:mid-low-noise-asymptotics-1} and \eqref{eq:mid-low-noise-asymptotics-2}.

\paragraph{Part III: Asymptotic expansion when $\sigma^2 \to +\infty$.} By the definition of $\bysest = \E[\theta \mid X ,y]$, we can make use of the following Pythagorean theorems
\begin{align*}
    \E_{y \sim \pi_{t}'} \brk{\norm{X \bysest - y}^2 \mid X} & = \E_{y \sim \pi_{t}'} \brk{\E_{\theta \mid X, y} \brk{\norm{X \theta - y}^2 - \norm{X (\theta - \bysest)}^2 \mid X, y}}, \\
    \E_{\theta \mid X, y} \brk{\norm{X (\theta - \bysest)}^2 \mid X, y} & = \E_{\theta \mid X, y} \brk{\norm{X \theta}^2 \mid X, y}  - \norm{X \bysest}^2,
\end{align*}
and obtain that
\begin{align}
    \tfunc(t) & = \frac{1}{n}\E_{y \sim \pi_{t}'} \brk{\norm{X \bysest - y}^2 \mid X} \nonumber \\
    & = \frac{1}{n} \E_{y \sim \pi_{t}'} \brk{\E_{\theta \mid X, y} \brk{\norm{X \theta - y}^2 - \norm{X \theta}^2 + \norm{X \bysest}^2\mid X, y}}  \nonumber \\
    & = \frac{1}{n} \E_{\theta \sim \pi, \tau \sim \normal(0, I_n), y=X\theta + \sigma\tau} \brk{\norm{X \theta -y}^2} - \frac{1}{n} \E_{\theta \sim \pi} \brk{\norm{X\theta}^2 \mid X} + \frac{1}{n} \E_{y \sim \pi_{t}'} \brk{\norm{X \bysest}^2 \mid X} \nonumber  \\
    & = t - \frac{1}{n} \E_{\theta \sim \pi, \tau \sim \normal(0, I_n), y=X\theta + \sigma\tau}  \brk{\norm{X (\theta - \bysest)}^2  \mid X} \nonumber  \\
    & = t - \frac{1}{n} \cdot \E \brk{\norm{X \theta}^2 \mid X}+ \frac{1}{n} \E\brk{\norm{X \bysest}^2  \mid X}. \label{eq:mid-high-noise-asymptotics-1} 
\end{align}
In the preceding displays, we marginalize over $y$ and $\theta$, using $y=X\theta + \sigma \tau$. The last two lines naturally give general lower and upper bounds as
\begin{align*}
    t - \mathsf{V}_{\pi} = t - \frac{1}{n} \cdot \E \brk{\norm{X \theta}^2 \mid X} \leq \tfunc(t) \leq t.
\end{align*}
Next, we derive the high noise asymptotics when $t \to +\infty$ based on the preceding displays. Recall that
\begin{align*}
    \bysest = \E \brk{\theta \mid X, y} = \frac{\int \theta \exp \brc{-\frac{\norm{X \theta - y}^2}{2 t}} \pi(\de \theta) }{\int \exp \brc{-\frac{\norm{X \theta - y}^2}{2 t}} \pi(\de \theta) }.
\end{align*}
We can then compute that
\begin{align*}
    &\E \brk{\bysest \bysest^\top} = \E_{y \sim \pi_{t}'}{\E \brk{\bysest \bysest^\top \mid X, y}} \nonumber \\
    &= \E_{y' \sim \pi', \tau \sim \normal(0, I_n)} \brk{ \frac{\iint \theta \tilde{\theta}^\top \exp \brc{-\frac{\norm{X \theta - y' - \sqrt{t}\tau}^2}{2 t}-\frac{\norm{X \tilde{\theta} - y' - \sqrt{t} \tau}^2}{2 t}} \pi(\de \theta)\pi(\de \tilde{\theta}) }{\prn{\int \exp \brc{-\frac{\norm{X \theta - y' - \sqrt{t} \tau}^2}{2 t}} \pi(\de \theta)}^2 } \,\,\,\Bigg| \,\,\, X}.
\end{align*}
Since $\int \norm{\theta}^2 \pi(d\theta) = \E[\norm{\theta}^2] < +\infty$, we can apply dominated convergence theorem to take limit under the integral
\begin{align*}
    & \lim_{t \to +\infty} \E_{y \sim \pi_t'}{\E \brk{\bysest \bysest^\top \mid X, y}} \\
    & = \E_{y' \sim \pi', \tau \sim \normal(0, I_n)} \brk{ \frac{\iint \theta \tilde{\theta}^\top \cdot \lim_{t \to +\infty}\exp \brc{-\frac{\norm{X \theta - y' - \sqrt{t} \tau}^2}{2 t}-\frac{\norm{X \tilde{\theta} - y' - \sqrt{t} \tau}^2}{2 t}} \pi(\de \theta)\pi(\de \tilde{\theta}) }{\prn{\int \lim_{t \to +\infty} \exp \brc{-\frac{\norm{X \theta - y' - \sqrt{t} \tau}^2}{2 t}} \pi(\de \theta)}^2 } \,\,\,\Bigg| \,\,\, X} \\
    & = \E_{\tau \sim \normal(0, I_n)} \brk{ \frac{\iint \theta \tilde{\theta}^\top \exp \brc{-\norm{\tau}^2} \pi(\de \theta)\pi(\de \tilde{\theta}) }{\exp \brc{-\norm{\tau}^2}  } \,\,\,\Bigg| \,\,\, X} = \E \brk{\theta} \E \brk{\theta}^\top = 0.
\end{align*}
Thus, $\bysest \stackrel{\mc{L}^2}{\rightarrow} 0$ by taking trace on the above limit, and
\begin{align*}
    \lim_{t \to +\infty} \frac{1}{n} \E_{y \sim \pi_t'}\brk{\norm{X \bysest}^2  \mid X} = \frac{1}{n} \Tr \prn{X \cdot \lim_{t \to +\infty} \E_{y \sim \pi_{t}'} \brk{\bysest \bysest^\top} \cdot X^\top } = 0.
\end{align*}
The last part of the proof is complete combining with Eq.~\eqref{eq:mid-high-noise-asymptotics-1}.

\subsection{Proof of Corollary~\ref{cor:trerr_too_high__thm:main}}
First, since $\frac{\sigma^4}{\jfunc_\pi^{-1}+\sigma^2} \leq \jfunc_\pi \sigma^4$,
we have
\[\trerr(\what{\theta}) \geq \jfunc_\pi\sigma^4 \ \Longrightarrow \ \trerr(\what{\theta})\geq \frac{\sigma^4}{\jfunc_\pi^{-1}+\sigma^2}.\]
Therefore, we can apply the bound~\eqref{eqn:trerr_too_high__thm:main} from Theorem~\ref{thm:main}. If $\trerr(\what{\theta}) > \jfunc_\pi\sigma^4$, then~\eqref{eqn:trerr_too_high__thm:main} immediately yields $\cost_\Sigma(\what{\theta})>0$. Moreover, if 
$\trerr(\what{\theta}) \geq C\cdot  \jfunc_\pi\sigma^4$ for some $C>1$, then by~\eqref{eqn:trerr_too_high__thm:main} we have
\begin{align*}
    \cost_\Sigma(\what{\theta})
    &\geq \lambda_\Sigma^{-1}\prn{\sqrt{\trerr(\what{\theta})}-\sqrt{\frac{\sigma^4}{\jfunc_\pi^{-1}+\sigma^2}}}^2\\
    &\geq \lambda_\Sigma^{-1}\prn{\sqrt{\trerr(\what{\theta})}-\sqrt{C^{-1}\cdot \trerr(\what{\theta})}}^2\textnormal{\quad since $\frac{\sigma^4}{\jfunc_\pi^{-1}+\sigma^2} \leq \jfunc_\pi\sigma^4\leq C^{-1}\cdot \trerr(\what{\theta})$}\\
    &=\lambda_\Sigma^{-1}(1-C^{-1/2})^2 \cdot \trerr(\what{\theta}).
\end{align*}

\subsection{Proof of Corollary~\ref{cor:trerr_too_low__thm:main}}
First, since $\frac{\sigma^4}{\mathsf{V}_\pi+\sigma^2} \geq \sigma^2 -  \mathsf{V}_\pi$,
we have
\[\trerr(\what{\theta}) \leq \sigma^2 - \mathsf{V}_\pi\ \Longrightarrow \ \trerr(\what{\theta})\leq \frac{\sigma^4}{\mathsf{V}_\pi+\sigma^2}.\]
Therefore, we can apply the bound~\eqref{eqn:trerr_too_low__thm:main} from Theorem~\ref{thm:main}. If $\trerr(\what{\theta}) <\sigma^2 -  \mathsf{V}_\pi$, then~\eqref{eqn:trerr_too_low__thm:main} immediately yields $\cost_\Sigma(\what{\theta})>0$. Moreover, if 
$\trerr(\what{\theta}) \leq \sigma^2 -  C\cdot \mathsf{V}_\pi$ for some $C>1$, then by~\eqref{eqn:trerr_too_low__thm:main} we have
\begin{align*}
    \cost_\Sigma(\what{\theta})
    &\geq \lambda_\Sigma^{-1}\prn{\sqrt{\frac{\sigma^4}{\mathsf{V}_\pi+\sigma^2}} - \sqrt{\trerr(\what{\theta})}}^2\\
    &\geq \lambda_\Sigma^{-1}\prn{\sqrt{\sigma^2 - \mathsf{V}_\pi} - \sqrt{\trerr(\what{\theta})}}^2\textnormal{\quad since $\frac{\sigma^4}{\mathsf{V}_\pi+\sigma^2} \geq \sigma^2 -  \mathsf{V}_\pi\geq \trerr(\what{\theta})$}\\
    &\geq \lambda_\Sigma^{-1}\prn{\sqrt{\sigma^2 - \mathsf{V}_\pi} - \sqrt{\sigma^2-\Delta}}^2\textnormal{\quad defining $\Delta = \sigma^2 - \trerr(\what{\theta})$}\\
    &\geq \lambda_\Sigma^{-1}\prn{\sqrt{\sigma^2 - C^{-1}\cdot \Delta} - \sqrt{\sigma^2 - \Delta}}^2\textnormal{\quad since $\trerr(\what{\theta}) \leq \sigma^2 -  C\cdot \mathsf{V}_\pi$ implies $\Delta\geq C\cdot \mathsf{V}_\pi$}\\
    &\geq \lambda_\Sigma^{-1}\prn{\frac{\Delta - C^{-1}\cdot \Delta}{2\sigma}}^2\textnormal{\quad since $\sqrt{\sigma^2-a} - \sqrt{\sigma^2-b} \geq \frac{b-a}{2\sigma}$ for $0\leq a\leq b\leq \sigma^2$}\\
    &=\frac{(4\lambda_\Sigma)^{-1}(1-C^{-1})^2}{\sigma^2} \cdot \Delta^2.
\end{align*}

\section{Proofs for Section~\ref{sec:applications}} \label{sec:proof-applications}

\subsection{Proof for the isotropic Gaussian prior} \label{proof:applications-isotropic}

We give the proof for the isotropic results first. Since $X\theta\mid X \sim \normal(0, XX^\top /d )$, we can denote by the $n$ eigenvalues of $\what{S}:=XX^\top/d$ by $\lambda_1 \geq \cdots \lambda_n > 0$. The classical results in random matrix theory enable the following almost sure limits (cf.~Lemma~\ref{lem:MP-BY-law} in Appendix~\ref{proof:technical}),
\begin{align*}
    &\lambda_1 \to  \prn{1 + \gamma^{-\half}}^2  =: \lambda^+_\gamma, \qquad \lambda_n \to   \prn{1 - \gamma^{-\half}}^2 =:\lambda^-_\gamma, \\
    &\mu_n(\de \lambda) := \frac{1}{n} \sum_{i=1}^n \delta_{\lambda_i} \cd\frac{\gamma}{2\pi} \cdot \frac{\sqrt{(\lambda^+_\gamma-\lambda)(\lambda - \lambda^-_\gamma)}}{\lambda} \ind_{\lambda^-_\gamma \leq \lambda \leq \lambda^+_\gamma} \de \lambda =: \mu_{\mathsf{MP}, \gamma}(\de \lambda). 
\end{align*}
Since $\Sigma = I$, the parameter $\lambda_\Sigma$ in Proposition~\ref{prop:compare_to_Bayes} then has the exact limit $\lambda_\Sigma = \gamma \lambda_1 \to (1 + \sqrt{\gamma})^2$. For $\mathsf{V}_{\pi} = \frac{1}{nd}\Tr(XX^\top) = \frac{1}{nd} \sum_{i=1}^n \sum_{j=1}^d X_{ij}^2$, by law of large numbers we have $\mathsf{V}_{\pi} \to 1$ with probability one. Alternatively, we can also make use of the semicircle integral $\int_a^b \sqrt{(b-\lambda)(\lambda-a)} \de \lambda = \frac{\pi(b-a)^2}{8}$ to confirm that
\begin{align*}
    \mathsf{V}_{\pi} & = \frac{1}{n}\Tr \prn{\what{S}} = \int \lambda \mu_n(\de \lambda) \to \int \lambda \mu_{\mathsf{MP}, \gamma}(\de \lambda)  = \int \frac{\gamma}{2\pi} \cdot \sqrt{(\lambda^+_\gamma-\lambda)(\lambda - \lambda^-_\gamma)}\ind_{\lambda^-_\gamma \leq \lambda \leq \lambda^+_\gamma} \de \lambda \nonumber \\
    & = \frac{\gamma}{2\pi} \cdot \frac{\pi ((1+\gamma^{-\half})^2- (1-\gamma^{-\half})^2)}{8} = 1.
\end{align*}
Finally, since $\pi' = \normal(0, \what{S})$, we can compute
\begin{align*}
    \jfunc_{\pi} & = \E \brk{\frac{1}{n}\norm{\nabla \log p(X\theta)}^2\mid X} = -\E \brk{\frac{1}{n} \Delta \log p(X\theta) \mid X} = \frac{1}{n} \Tr \prn{\what{S}^{-1}} \nonumber \\
    & = \int \frac{1}{\lambda} \mu_n(\de \lambda) \to \int \frac{1}{\lambda} \mu_{\mathsf{MP}, \gamma}(\de \lambda) = \frac{\gamma}{\gamma-1}.
\end{align*}
The last equality is from the resolvent identity at $z \to i\cdot0+$~\cite[Lemma~3.11]{bai2010spectral}.

For the general sequence of $\Sigma(n)$, we refer the reader for backgrounds in Stieltjes transform and random matrix theory to Lemmas~\ref{lem:generalized-MP-BY-law} and \ref{lem:edge-equations} that are necessary for the general version of Proposition~\ref{prop:applications-isotropic} below.

\begin{proposition} \label{prop:applications-isotropic-general-Sigma} Let Assumption~\ref{assmp:proportional-asymptotics}. Under the setups of Lemmas~\ref{lem:generalized-MP-BY-law} and \ref{lem:edge-equations}, for $\sigma^2 > 0$, we have almost surely
\begin{align*}
    &\mathsf{V}_{\pi} \to \int \lambda \mu(\de \lambda), \qquad \jfunc_{\pi} \to \int \frac{1}{\lambda} \mu(\de \lambda) \qquad \lambda_{I} \to \gamma \lambda_\mu^+,
\end{align*}
where $\mu$ is the generalized Marchenko-Pastur limit for the sequences of covariance matrices $\Sigma(n)$ whose empirical spectrum distributions $\nu_n$ converge to some $\nu$ weakly.
\end{proposition}
The proof for Prop.~\ref{prop:applications-isotropic-general-Sigma}, is identical to the preceding proof for $\Sigma = I$, and thus we omit the repetitive details. For $\lambda_\Sigma$, as it equals to $\frac{1}{n} \norm{\Sigma^{-\half} \cdot X^\top X \Sigma^{-\half}} = \frac{1}{n} \norm{Z^\top Z}$, the isotropic result applies and $\lambda_\Sigma \to (1+\sqrt{\gamma})^2$.

\subsection{Proofs for the approximately-low-rank Gaussian prior} \label{proof:applications-perturbed-low-rank}

\subsubsection{Proof of Proposition~\ref{prop:applications-perturbed-low-rank-positive-eta}} \label{proof:applications-perturbed-low-rank-positive-eta}
\paragraph{Part I: Limits of the parameters.} Let $\what{S}=X \Omega X^\top$ and its e.s.d.\ be $\mu_n$. Since $X\theta\mid X \sim \normal(0, \what{S})$, the convergence of $\mathsf{V}_{\pi}$ follows from law of large numbers and
\begin{align*}
    \E_X \brk{\mathsf{V}_{\pi}} = \frac{1}{n} \Tr(\E_{X}\brk{\what{S}}) = \frac{1}{d} \Tr(\Sigma) = 1.
\end{align*}
As $\lambda_\Sigma = \frac{1}{n} \norm{X^\top X}$, the same limit as in Proposition~\ref{prop:applications-isotropic} hold. It remains only nontrivial to compute $\jfunc_{\pi}$. Denote by the e.s.d.\ of $d\Omega_n$ by $\nu_n$, it is clear that
\begin{align*}
    \nu_n \cd \nu := \frac{\rho}{\gamma} \delta_{\gamma/\rho + \eta} + \prn{1 - \frac{\rho}{\gamma}} \delta_{\eta}.
\end{align*}
We can then apply Lemma~\ref{lem:generalized-MP-BY-law} since $\what{S} = X(d\Omega) X^\top / d$, which implies
\begin{align*}
    \frac{1}{m_\mu} = -z + \int \frac{\lambda}{1 + \lambda m_\mu \gamma^{-1}} \nu(\de \lambda) = -z + \frac{1 + \eta \rho \gamma^{-1}}{1 +  m_\mu (\rho^{-1} + \eta \rho^{-1})} + \frac{\eta - \eta \rho \gamma^{-1}}{1 + m_\mu \eta \gamma^{-1}}.
\end{align*}
Set $z=0$, this solves the quadratic equation
\begin{align*}
    \eta(\gamma-1) m_\mu^2+\underbrace{\brk{\eta^2(\gamma-\rho-1)+\eta\left(\gamma \rho-\rho^2-\rho\right)+\gamma(\rho-1)}}_{=C} m_\mu-\gamma(\eta+\rho) = 0,
\end{align*}
which yields for the positive branch solution:
\begin{align*}
    m_\mu(0)=\frac{-C+\sqrt{C^2+4 \eta(\gamma-1) \gamma(\eta+\rho)}}{2 \eta(\gamma-1)}.
\end{align*}
We thus have the explicit form for the limit of $\jfunc_\pi$ by identifying
\begin{align*}
    \jfunc_\pi= -\E \brk{\frac{1}{n} \Delta \log p(X\theta) \mid X} = \frac{1}{n} \Tr \prn{\what{S}^{-1}} = \int \frac{1}{\lambda} \mu_n(\de \lambda) \to\int \frac{1}{\lambda} \mu(\de \lambda) = m_{\mu}(0).
\end{align*}
In this step we invoke Lemma~\ref{lem:generalized-MP-BY-law} implicitly, which allows us to restrict ourselves to the compact interval $[\lambda_\mu^{-}-\epsilon, \lambda_\mu^+ + \epsilon]$, since the extreme eigenvalues converge almost surely. Consequently, the function $\lambda \mapsto 1/\lambda$ is bounded on this interval for a suitable choice of $\epsilon$, and thus weak convergence applies for the function $1/\lambda$.

\paragraph{Part II: Asymptotics for $\eta \to 0+$.} 
We compute the asymptotics for $m_\mu(0)$ as the limit of $\jfunc_{\pi}$.
\subparagraph{Case I: $\rho > 1$.} In this case $C(\gamma, \rho, 0+)$ has a positive limit $\gamma(\rho-1)$, and thus
\begin{align*}
    m_{\mu}(0) & = \frac{-C+\sqrt{C^2+4 \eta(\gamma-1) \gamma(\eta+\rho)}}{2 \eta(\gamma-1)} = \frac{2\gamma(\eta+\rho)}{C+\sqrt{C^2+4 \eta(\gamma-1) \gamma(\eta+\rho)}} \nonumber \\
    & \to \frac{\gamma \rho}{C(\gamma, \rho, 0+)} = \frac{\rho}{\rho-1}. 
\end{align*}
\subparagraph{Case II: $\rho = 1$.} When $\rho=1$, taking $C(\gamma, 1, \eta) = \eta(\gamma-2) + o(\eta)$ into the above display gives
\begin{align*}
    m_{\mu}(0) & = \frac{2\gamma \rho + o(\eta)}{o(\eta^\half)+\sqrt{4 \eta(\gamma-1) \gamma \rho + o(\eta)}} = \sqrt{\frac{\gamma}{\eta(\gamma-1)}} + o(\eta^{-\half}).
\end{align*}
\subparagraph{Case III: $0<\rho < 1$.} In this case $C(\gamma, \rho, 0+)$ has a negative limit $\gamma(\rho-1)$, and thus
\begin{align*}
    m_{\mu}(0) & = \frac{-C+\sqrt{C^2+4 \eta(\gamma-1) \gamma(\eta+\rho)}}{2 \eta(\gamma-1)} = \frac{2\gamma(1-\rho) + o(\eta)}{2 \eta(\gamma-1)} = \frac{\gamma(1-\rho)}{\eta(\gamma-1)} + o(\eta^{-1}).
\end{align*}
The proof is complete.

\subsubsection{Proof of Proposition~\ref{prop:applications-perturbed-low-rank-zero-eta}} \label{proof:applications-perturbed-low-rank-zero-eta}
We divide this section into two parts. The main proof of the proposition is in only Part I. In Part II, we analyze the training error for $\eta > 0$ and derive asymptotic solutions under the scaling regime $\sigma = C\sqrt{\eta}$, for some constant $C$, in the limit as $\eta \to 0+$.
\paragraph{Part I: The unperturbed prior $\eta = 0$.} We begin by examining the case $\eta = 0$, as it offers conceptual clarity while avoiding substantial technical complications. It is equivalent to having $X \in \R^{n \times r}$, $r/n \to \rho$, $\Sigma=I_r$ and $\theta \sim \normal(0, I_r/r)$. The posterior mean (i.e.\ the Bayes estimator) for $y \sim \normal(0, XX^\top/r + \sigma^2 I_r)$ is then $\bysest = (X^\top X+\sigma^2 rI_r)^{-1} X^\top y$, and therefore
\begin{align*}
    \trerr(\bysest) & = \E \brk{\frac{1}{n} \norm{X \bysest -y}^2} \nonumber \\
    & = \frac{1}{n} \E \brk{ \norm{X(I - (X^\top X + \sigma^2 r I_r)^{-1} X^\top X) \theta}^2} + \frac{1}{n} \E \brk{ \norm{\prn{I -X(X^\top X + \sigma^2 r I_r)^{-1} X^\top} \epsilon}^2} \nonumber \\
    & = \frac{\sigma^4}{n} \Tr \prn{(XX^\top/r + \sigma^2 I_n)^{-2} (XX^\top/r)} + \frac{\sigma^6}{n} \Tr \prn{(XX^\top/r + \sigma^2 I_n)^{-2}} \nonumber \\
    & = \int \frac{\sigma^4\lambda}{(\lambda + \sigma^2)^2} + \frac{\sigma^6}{(\lambda + \sigma^2)^2} \mu_n(\de \lambda) \nonumber \\
    & \to \int \frac{\sigma^4}{\lambda + \sigma^2}\mu_{\mathsf{MP, \rho}}(\de \lambda)\nonumber \\
    &= \sigma^4 m_{\mu_{\mathsf{MP, \rho}}}(-\sigma^2) = \sigma^4 \cdot \frac{-\left(\sigma^2 \rho-\rho+1\right) + \sqrt{\left(\sigma^2 \rho-\rho+1\right)^2-4 \sigma^2 \rho}}{2 \sigma^2}.
\end{align*}
In the last line we use Lemma~\ref{lem:generalized-MP-BY-law} for $\nu = \delta_1$ and $z=-\sigma^2$. As it holds for $m = m_{\mu_{\mathsf{MP, \rho}}}(-\sigma^2)$ that
\begin{align*}
\frac{1}{m} = \sigma^2 + \frac{1}{1 + m/\rho}.
\end{align*}
Using the asymptotics for $\rho > 1$ that
\begin{align*}
    \frac{-\left(\sigma^2 \rho + \rho - 1\right) + \sqrt{\left(\sigma^2 \rho + \rho - 1\right)^2+4 \sigma^2 \rho}}{2 \sigma^2} & = \frac{4\sigma^2 \rho}{2 \sigma^2 \cdot \prn{\left(\sigma^2 \rho + \rho - 1\right) + \sqrt{\left(\sigma^2 \rho + \rho - 1\right)^2+4 \sigma^2 \rho}}} \nonumber \\
    & = \frac{\rho}{\rho-1} + O(\sigma^2),
\end{align*}
and for $\rho < 1$ that
\begin{align*}
    \frac{-\left(\sigma^2 \rho + \rho - 1\right) + \sqrt{\left(\sigma^2 \rho + \rho - 1\right)^2+4 \sigma^2 \rho}}{2 \sigma^2} & = \frac{2(1-\rho) -\sigma^2 \rho + \frac{\rho(1+\rho)}{1-\rho} \sigma^2 + O(\sigma^4)}{2 \sigma^2} \nonumber \\
    & = \frac{1-\rho}{\sigma^2} + \frac{\rho^2}{1-\rho} + O(\sigma^2).
\end{align*}
we complete the proof for $\eta = 0$.
\paragraph{Part II: The general training error when $\eta > 0$.} Similar to the previous calculations by considering the exact Bayes estimator from the posterior mean of $y \sim \normal(0, X\Omega X^\top + \sigma^2 I_d)$
\begin{align*}
\bysest = \Omega X^\top (X \Omega X^\top+\sigma^2 I_d)^{-1} y.
\end{align*}
Let $\mu_n$ be the empirical spectrum distribution of $X \Omega X^\top$, we then have the exact formula for the Bayes error as (using the same calculations in the first part):
\begin{align*}
    \trerr(\bysest) & = \int \frac{\sigma^4}{\lambda + \sigma^2} \mu_n(\de \lambda) \to \int \frac{\sigma^4}{\lambda + \sigma^2} \mu(\de \lambda) = \sigma^4 m_\mu(-\sigma^2),
\end{align*}
where $\mu$ is the generalized Marchenko-Pastur limit by Lemma~\ref{lem:generalized-MP-BY-law} and $m_\mu(z)$ is its associated Stieltjes transform. Substituting $z = -\sigma^2$ into Lemma~\ref{lem:generalized-MP-BY-law} yields for $m_{\eta} = m_\mu(-\sigma^2)$,
\begin{align*}
    \frac{1}{m_\eta} = \sigma^2 + \frac{1 + \eta \rho \gamma^{-1}}{1 +  m_\eta (\rho^{-1} + \eta \rho^{-1})} + \frac{\eta - \eta \rho \gamma^{-1}}{1 + m_\eta \eta \gamma^{-1}}.
\end{align*}
When $\rho<1$, we substitute in the ansatz $m_\eta = C'/\eta + o(\eta^{-1})$ under the noise asymptotics $\sigma = C \sqrt{\eta}$, the equation simplifies into
\begin{align*}
    \frac{1-\rho}{C'}=C^2+\frac{\gamma-\rho}{\gamma+C'},
\end{align*}
which admits a positive solution $C'=C'(C)$ with $C'(0)= \frac{\gamma(1-\rho)}{\gamma-1}$.

\subsection{Proof for the mixture of approximately-sparse priors} \label{proof:applications-perturbed-sparse}
We provide the proofs in three parts. In the first part we upper bound the Fisher information by convexity. In the second part we utilize the isoperimetric lower bound from Lemma~\ref{lem:EII}.

\paragraph{Part I: Upper bound by convexity of Fisher information.} For a mixture of density $p(y) = \sum_{i=1}^m \alpha_i p_i(y)$ where Fisher information exists for each individual $p_i$. We can upper bound pointwise by Cauchy-Schwarz inequality:
\begin{align*}
    \norm{\nabla p(y)}^2 = \norm{\sum_{i=1}^m \alpha_i \nabla p_i(y)}^2 \leq \prn{\sum_{i=1}^m \alpha_i p_i(y)} \cdot \prn{ \sum_{i=1}^m \alpha_i \frac{\norm{\nabla p_i(y)}^2}{p_i(y)}}= p(y) \cdot \prn{ \sum_{i=1}^m \alpha_i \frac{\norm{\nabla p_i(y)}^2}{p_i(y)}}.
\end{align*}
Apply the pointwise bound to the mixture considered in the example with $m=\binom{d}{K}$, $\alpha_S = \binom{d}{K}^{-1}$ and $p_S = \phi(y; 0, X \Omega_S X^\top)$ and we have
\begin{align*}
    \frac{\norm{\nabla p(y)}^2}{p(y)} \leq \sum_{S} \alpha_S  \frac{\norm{\nabla p_S(y)}^2}{p_S(y)} = \frac{1}{\binom{d}{K}} \sum_S \frac{\norm{\nabla \phi(y; 0, X\Omega_S X^\top)}^2}{\phi(y; 0, X\Omega_S X^\top)}.
\end{align*}
Integrating over $y$ yields
\begin{align*}
    \jfunc_{\pi} &= \int \frac{\norm{\nabla p(y)}^2}{p(y)} \de y \leq \int \frac{1}{\binom{d}{K}} \sum_S \frac{\norm{\nabla \phi(y; 0, X\Omega_S X^\top)}^2}{\phi(y; 0, X\Omega_S X^\top)} \de y \nonumber \\
    & =  \frac{1}{\binom{d}{K}} \sum_S \int \norm{\nabla \log \phi(y; 0, X\Omega_S X^\top)}^2\phi(y; 0, X\Omega_S X^\top) \de y \\
    & = \frac{1}{\binom{d}{K}} \sum_{S} \E_{y \sim \normal(0,X \Omega_S X^\top)} \brk{\norm{\nabla \log \phi(y; 0, X\Omega_S X^\top)}^2}   \\
    &= \frac{1}{\binom{d}{K}} \sum_S \Tr \prn{\prn{X\Omega_S X^\top}^{-1}} \leq \frac{1}{\eta} \Tr \prn{(XX^\top/d)^{-1}}.
\end{align*}
In the last inequality, we use the uniform lower bound $\Omega_S \succeq \eta I_d/d$. Let $\mu_n$ be the e.s.d.\ of $XX^\top/d$, $\mu_n \cd \mu_{\mathsf{MP}, \gamma}$ by Lemma~\ref{lem:MP-BY-law}, and by the resolvent identity at $z \to i \cdot 0+$ in~\cite[Lemma~3.11]{bai2010spectral} we conclude
\begin{align*}
    \jfunc_{\pi}\leq \frac{1}{\eta} \Tr \prn{(XX^\top/d)^{-1}} =\frac{1}{\eta}\int \frac{1}{\lambda} \mu_n(\de \lambda) \to \frac{1}{\eta}\int \frac{1}{\lambda} \mu_{\mathsf{MP}, \gamma}(\de \lambda) \to \frac{\gamma}{\eta(\gamma-1)}.
\end{align*}
We thus obtain the upper bound
\begin{align*}
    \limsup_{n \to +\infty} \jfunc_\pi \leq \frac{\gamma}{\eta(\gamma-1)}.
\end{align*}
\paragraph{Part II: Lower bound by entropic isoperimetric inequality.} We begin by introducing the necessary notation that we will leverage to derive upper and lower bounds for $\jfunc_\pi$. Define the differential entropy for $\mathsf{Ent}_\pi$ for $y=X\theta$ by $\mathsf{Ent}_\pi = \E \brk{-\log p(X\theta)} > 0$ (by Jensen's inequality). With a slight abuse of notations, we also denote $S$ by the discrete random variable uniformly drawn from all $\binom{d}{k}$ subsets of size $K$ of $[d]$. Then we can view the distribution of $y$ as drawing $S$ first and sampling $y \mid S \sim \normal(0, X\Omega_S X^\top)$. We can then define the conditional entropy and mutual information as follows:
\begin{align*}
    \mathsf{Ent}_{\pi \mid S} = \E_S \brk{\E_{y \sim \pi \mid S} \brk{-\log \phi(y \mid S;0, X\Omega_SX^\top)}}, \qquad \mathscr{I}(y; S) = \mathsf{Ent}_{\pi}-\mathsf{Ent}_{\pi \mid S}.
\end{align*}
Combine the following explicit formula for the differential entropy for multivariate Gaussians, and the familiar entropy bound for the mutual information of a discrete random variable,
\begin{align*}
    \mathsf{Ent}_{\pi \mid S} & = \frac{1}{\binom{d}{K}} \sum_S \prn{\frac{n}{2} \log 2 \pi  + \frac{n}{2} + \frac{1}{2} \det X \Omega_S X^\top} = \frac{n}{2} + \frac{n}{2} \log 2 \pi + \frac{1}{2} \frac{1}{\binom{d}{K}} \sum_S \log \det \prn{X \Omega_S X^\top}, \\
    \mathscr{I}(y; S) & \leq \mathscr{H}(S) = \log \binom{d}{K} \stackrel{\mathrm{(i)}}{\leq} K \log \frac{d}{K} + K, 
\end{align*}
where we use the Stirling's bound $K! \geq (K/e)^K$ in (i); we can thus upper bound
\begin{align*}
    \mathsf{Ent}_{\pi} & =\mathsf{Ent}_{\pi \mid S} + \mathscr{I}(y; S)  \nonumber \\
    & \leq \frac{n}{2} + \frac{n}{2} \log 2 \pi + K \log \frac{d}{K} + K + \frac{1}{2} \frac{1}{\binom{d}{K}} \sum_S \log  \det \prn{X \Omega_S X^\top}.
\end{align*}
We are now ready to apply the entropic isoperimetric inequality in Lemma~\ref{lem:EII} on $\jfunc_{\pi}$, which implies
\begin{align*}
    \jfunc_{\pi} = \frac{1}{n} \cdot \E \brk{\norm{\nabla \log p(y)}^2} \geq 2 \pi e \cdot \exp \brc{-\frac{2}{n} \E \brk{-\log p(y)}} = \exp\brc{\log (2\pi) + 1 - \frac{2}{n} \mathsf{Ent}_{\pi}}.
\end{align*}
Comibining with the entropy bound yields
\begin{align}
    \jfunc_{\pi} \geq \exp \brc{-\frac{2K}{n} \log \frac{d}{K} - \frac{2K}{n} - \frac{1}{\binom{d}{K}} \sum_S \frac{1}{n} \log \det (X\Omega_S X^\top) }. \label{eq:mixture-sparse-mid-1}
\end{align}
Utilizing the special structure of $\Omega_S$, we can further upper bound the determinant by the following lemma.
\begin{lemma} \label{lem:log-det-upper-bound}
    For all $S \in [d], |S|=K$, we have
    \begin{align*}
        \frac{1}{n} \log \det (X\Omega_S X^\top) \leq \prn{1 - \frac{K}{n}}\log \eta + \frac{K}{n} \log \frac{d}{K} + \frac{1}{n} \log \det \prn{\frac{XX^\top}{d}}.
    \end{align*}
\end{lemma}
\begin{proof}
    Denote by $X_S$ the $K$ columns of $X$ whose indices belong to $S$. Then $X\Omega_S X^\top = \frac{1-\eta}{K} X_S X_S^\top + \frac{\eta}{d} XX^\top$. Making use of the determinant identity $\det(I_p + AB^\top) = \det (I_q + B^\top A)$ for any $A, B \in \R^{p \times q}$, it follows that
    \begin{align*}
        \det \prn{X\Omega_S X^\top } & = \det \prn{\frac{1-\eta}{K} X_S X_S^\top + \frac{\eta}{d} XX^\top} \nonumber \\
        & = \eta^n \cdot \det \prn{\frac{XX^\top}{d}} \cdot \det \prn{I_d + \frac{(1-\eta)d}{\eta K} (XX^\top)^{-\half} X_S X_S^\top (XX^\top)^{-\half}}  \\
        & = \eta^n \cdot \det \prn{\frac{XX^\top}{d}} \cdot \det \prn{I_K + \frac{(1-\eta)d}{\eta K} X_S^\top (XX^\top)^{-1} X_S } \\
        & \leq \eta^n \cdot \det \prn{\frac{XX^\top}{d}} \cdot \prn{1 + \frac{(1-\eta)d}{\eta K}}^K,
    \end{align*}
    where in the last inequality, we use $\norm{X_S^\top (XX^\top)^{-1} X_S} \leq \norm{X_S^\top (X_SX_S^\top)^{\dagger} X_S} = 1$. Taking logarithms on both sides, and applying the inequality $\eta K + (1-\eta)d \leq d$, we complete the proof.
\end{proof}

Applying Lemma~\ref{lem:log-det-upper-bound} to Eq.~\eqref{eq:mixture-sparse-mid-1}, we therefore obtain
\begin{align*}
    \jfunc_{\pi} & \geq \exp \brc{-\prn{1 - \frac{K}{n}}\log \eta -\frac{3K}{n} \log \frac{d}{K} - \frac{2K}{n} - \frac{1}{n} \log \det \prn{\frac{XX^\top}{d}} }.
\end{align*}
We can once more exploit the convergence to the Marchenko–Pastur law as in the previous part, so that
\begin{align*}
\frac{1}{n} \log \det \prn{\frac{XX^\top}{d}} = \int \log \lambda \cdot \mu_n(\de \lambda) \to \int \log \lambda \cdot \mu_{\mathsf{MP}, \gamma} (\de \lambda) = \prn{1-\gamma} \cdot \log \prn{1-\gamma^{-1}} -1.
\end{align*}
The last equality is from~\cite[Example 1.1.1 \& Sec.~9.12.3]{bai2010spectral}. We thus conclude the proof, additionally using $K=o(n)$ and $x\log x \to 0$ for $x \to 0+$,
\begin{align*}
    \liminf_{n \to +\infty} \jfunc_{\pi} \geq \frac{1}{\eta} \cdot e \prn{1-\frac{1}{\gamma}}^{\gamma - 1}.
\end{align*}
Indeed, the lower bound holds for general $K$ by defining $\alpha = K/n \in[0, 1]$ and
\begin{align*}
    \liminf_{n \to +\infty} \jfunc_{\pi} \geq \prn{\frac{1}{\eta}}^{1-\alpha} \cdot \prn{\frac{\alpha}{\gamma}}^{3\alpha} \cdot e^{1-2\alpha} \prn{1-\frac{1}{\gamma}}^{\gamma - 1}.
\end{align*}

\section{Proofs for Section~\ref{sec:fine-grained}}

\subsection{Proof of Lemma~\ref{lem:bayes-train-err-representation}} \label{proof:bayes-train-err-representation}
First, we remark that $p$ does not need to be differentiable for this Lemma to hold, as Gaussian kernels smooth $\mc{L}^1$ functions into smooth functions for any $\sigma^2 > 0$---that is, $p_{\sigma^2}$ is smooth for any $\sigma^2>0$, for any density $p$.

It follows from $\frac{\partial}{\partial y} \exp \brc{-\frac{\norm{X \theta' - y}^2}{2 \sigma^2}} = \frac{X \theta' -y}{\sigma^2} \exp \brc{-\frac{\norm{X \theta' - y}^2}{2\sigma^2}}$ and we can interchange the order of differentiation and integration since the derivative is uniformly bounded, $\norm{\frac{X \theta' -y}{\sigma} \exp \brc{-\frac{\norm{X \theta' - y}^2}{2\sigma^2}}} \leq \sup_{t \in [0,+\infty)} t e^{-t^2/2} \leq e^{-1/2}$. To be specific,
    \begin{align*}
        \frac{\partial}{\partial y} \log \E_{\theta' \sim \pi} \brk{\exp \brc{-\frac{\norm{X \theta' - y}^2}{2 \sigma^2}} \mid X, y} & = \frac{\int \frac{\partial}{\partial y} \exp \brc{-\frac{\norm{X \theta' - y}^2}{2 \sigma^2}} \pi(\de \theta') }{\int \exp \brc{-\frac{\norm{X \theta' - y}^2}{2 \sigma^2}} \pi(\de \theta') }  = \frac{X \bysest - y}{\sigma^2}.
    \end{align*}
    Eq.~\eqref{eq:bayes-train-err-cgf} holds by substituting $y=X\theta + \sigma \tau$. Eq.~\eqref{eq:bayes-train-err-fisher} holds given the existence of $p$ since for $y'=X\theta'$,
    \begin{align*}
    \exp \brc{-\frac{\norm{X \theta' - y}^2}{2 \sigma^2}} \pi(\de \theta) & = \exp \brc{-\frac{\norm{y' - y}^2}{2 \sigma^2}} p(y') \de y' \nonumber \\
    & = C \phi_{\sigma^2}(y-y') p(y') \de y' \nonumber \\
    (X \theta' - y) \cdot \exp \brc{-\frac{\norm{X \theta' - y}^2}{2 \sigma^2}} \pi(\de \theta) & = (y' - y) \cdot \exp \brc{-\frac{\norm{y' - y}^2}{2 \sigma^2}} p(y') \de y' \nonumber \\
    & = \sigma^2 \frac{\partial}{\partial y} \exp \brc{-\frac{\norm{y-y'}^2}{2 \sigma^2}} p(y') \de y' \nonumber \\
    & = C\sigma^2 \nabla \phi_{\sigma^2}(y-y') p(y') \de y',
    \end{align*}
    where $C = (2\pi \sigma^2)^{-\frac{n}{2}}$ is the integration constant. This confirms
    \begin{align*}
        \frac{\int \frac{\partial}{\partial y} \exp \brc{-\frac{\norm{X \theta' - y}^2}{2 \sigma^2}} \pi(\de \theta') }{\int \exp \brc{-\frac{\norm{X \theta' - y}^2}{2 \sigma^2}} \pi(\de \theta') } = \frac{\int \sigma^2 \nabla  \phi_{\sigma^2}(y-y') p(y') \de y' }{\int \phi_{\sigma^2}(y-y') p(y') \de y' } = \frac{\sigma^2 \nabla  p_{\sigma^2}(y) }{p_{\sigma^2}(y) }  = \sigma^2 \nabla \log p_{\sigma^2}(y).
    \end{align*}

\subsection{Proof of Proposition~\ref{prop:monotonicity}} \label{proof:monotonicity}
Let $t=\sigma^2$. The monotonicity of $\tfunc(t)/t^2 = \jfunc(t)$ follows from Stam's Fisher information inequality in Lemma~\ref{lem:FII}, which implies $\jfunc(t)^{-1} \geq \jfunc(0)^{-1} +t$. The rest of the proof is for the monotonicity of $\tfunc(t)$ itself. This requires using the alternative form in Eq.~\eqref{eq:mid-high-noise-asymptotics-1},
\begin{align*}
    \tfunc(t) & = t - \frac{1}{n} \E_{y' \sim \pi', \tau \sim \normal(0, I_n)}  \brk{\norm{y' - X \bysest}^2  \mid X} \nonumber \\
    &= t - \frac{1}{n} \E_{y' \sim \pi', \tau \sim \normal(0, I_n)} \E \brk{\norm{y' - \E \brk{y' \mid y}}^2  \mid y = y' + \sqrt{t} \tau} \nonumber \\
    & =: t - \mmse(1/t),
\end{align*}
identifying the formulation of MMSE in Gaussian channels. When $n=1$, the explicit gradient $\mmse'(s)$ under the time inversion $s = 1/t$ is in \cite[Cor.~2]{guo2011estimation}, given as $\mmse'(s) = -\E \brk{\Cov (y' \mid y)^2 }$. In the vector version, a similar formula indeed holds applying the general multi-dimensional derivative identity~\cite[Prop.~1]{dytso2020general}, which implies the following explicit equation.
\begin{lemma} \label{lemma:mmse-grad} For all $s \in [0, +\infty)$ and $y = y'+ \sqrt{1/s}\; \tau$,
$$\mmse'(s) = - \frac{1}{n} \cdot \E \brk{\norm{\Cov (y' \mid y)}_F^2}.$$
\end{lemma}
\begin{proof}
    Given the Jacobian identity from~\cite[Prop.~1]{dytso2020general}, we have 
    \begin{align*}
        \frac{\partial \E \brk{y' \mid y}}{\partial y} = s\Cov (y' \mid y),
    \end{align*}
    and thus
    \begin{align*}
        \mmse'(s) & = \frac{1}{n} \cdot \E_y\brk{\E \brk{2\prn{y' - \E \brk{y' \mid y}}^\top \prn{-\frac{\partial \E \brk{y' \mid y}}{\partial y}}  \mid y} \cdot \frac{\de y}{\de s}}\\
        & = \frac{1}{n} \cdot \E_y\brk{\E \brk{2\prn{y' - \E \brk{y' \mid y}}^\top \prn{-s \Cov(y' \mid y)}  \mid y} \cdot \prn{-\frac{1}{2} s^{-3/2}} \tau}\\
        & = \frac{1}{n} \cdot \E \brk{\Tr \prn{\E \brk{\tau/\sqrt{s} \prn{y' - \E \brk{y' \mid y}}^\top \mid y} \cdot \Cov(y' \mid y)} }.
    \end{align*}
    We then use that $\tau/\sqrt{s} = y-y'$ and conditional on $y$,
    \begin{align*}
        \E \brk{\tau/\sqrt{s} \prn{y' - \E \brk{y' \mid y}}^\top \mid y} = -\E \brk{\prn{y' - \E \brk{y' \mid y}}\prn{y' - \E \brk{y' \mid y}}^\top \mid y} = - \Cov(y' \mid y),
    \end{align*}
    and the proof is done by
    \begin{align*}
    \mmse'(s) & = - \frac{1}{n} \cdot \E \brk{\Tr \prn{\Cov(y' \mid y)^2}} = -\frac{1}{n} \cdot \E \brk{\norm{\Cov (y' \mid y)}_F^2}.
    \end{align*}
\end{proof}

Given the derivative formula, by a variable transformation, we have
\begin{align*}
    \tfunc'(t) & = 1 - \mmse'(1/t) \cdot \frac{\de }{\de t} \prn{\frac{1}{t}} = 1 - \frac{1}{nt^2} \cdot \E \brk{\norm{\Cov (y' \mid y)}_F^2 } \\
    & = 1 - \frac{1}{nt^2} \cdot \E \brk{\norm{\E \brk{\prn{y' - \E \brk{y' \mid y}}\prn{y' - \E \brk{y' \mid y}}^\top}}_F^2 } \\
    & \stackrel{\mathrm{(i)}}{\geq} 1 - \frac{1}{nt^2} \cdot \E \brk{\norm{\E \brk{(y' - y)(y' - y)^\top}}_F^2 } \nonumber \\
    & = 1 - \frac{1}{nt^2} \cdot \E \brk{\norm{t \E\brk{\tau \tau^\top}}_F^2 } = 1- \frac{1}{nt^2} \cdot nt^2 = 0,
\end{align*}
where in (i) we use the simple fact that conditional on $y$,
\begin{align*}
    &\E \brk{\prn{y' - \E \brk{y' \mid y}}\prn{y' - \E \brk{y' \mid y}}^\top \mid y} + \prn{y - \E \brk{y' \mid y}}\prn{y - \E \brk{y' \mid y}}^\top \nonumber \\
    &= \E \brk{(y' - y)(y' - y)^\top \mid y}.
\end{align*}
Then $\tfunc'(t) \geq 0$ and $\tfunc$ does not decrease in $[0, +\infty)$. The proof is complete.

\subsection{Training error formula for the numerical simulation} \label{sec:simulation-calculations}
In this section we give details for calculating the curves shown in Figure~\ref{fig:monotonic}.
Let $t= \sigma^2$. Since $\pi' = \frac{1}{2} \normal(-1, \eta) + \frac{1}{2} \normal(1, \eta)$, one has $\pi'_t = \frac{1}{2} \normal(-1, \eta+t) + \frac{1}{2} \normal(1, \eta+t)$. In this case, we can compute the explicit formula for its Fisher information $\jfunc(t)$ from
\begin{align*}
    \jfunc(t) & = \E_{y \sim \pi'_t} \brk{\brc{\frac{\de}{\de y} \log \prn{\frac{1}{2} \phi_{\eta + t}(y-1) + \frac{1}{2} \phi_{\eta+ t}(y+1) } }^2} \\
    & = \frac{1}{(\eta + t)^2} \cdot \E_{y \sim \pi'_t} \brk{\prn{\frac{(y-1) \phi_{\eta+ t}(y-1)+(y+1) \phi_{\eta+ t}(y+1)}{\phi_{\eta+ t}(y-1)+\phi_{\eta+ t}(y+1)}}^2} \nonumber \\
    & = \frac{1}{(\eta + t)^2} \cdot \E_{y \sim \pi'_t} \brk{\prn{y - \frac{\exp \brc{\frac{2y}{\eta + t}} - 1}{\exp \brc{\frac{2y}{\eta + t}} + 1}}^2} = \frac{1}{\eta+t}-\frac{1}{(\eta+t)^2} \mathbb{E}_{\tau \sim \normal(0,1)}\left[\frac{4 \exp \left(\frac{2(1+\sqrt{\eta+t} \tau)}{\eta+t}\right)}{\left(1+\exp \left(\frac{2(1+\sqrt{\eta+t} \tau)}{\eta+t}\right)\right)^2}\right].
\end{align*}
At each value of $(\eta, t)$, we can run Monte-Carlo to evaluate $\jfunc(t)$.

\section{Higher order asymptotics for the Bayes training error} \label{sec:higher-order}
In this section, we provide higher order asymptotics of $\trerr(\sigma^2)$ when $\sigma^2 \to 0+$. We will use $t=\sigma^2$ in this section. As partially raised in Remark~\ref{remark:suggested-asymptotics}, the main technical hurdle that prevents us from directly taking limits such as $\lim_{t \to 0+} \ifunc(t) = \ifunc(0)$ hinges on the general insufficient regularities in the underlying noiseless data distribution $\pi'$ along with its associated density $p$ on $\R^n$. 

We will need to introduce some necessary notations to facilitate delivering our results. On $\R^n$, we use the shorthand $\partial_i := \partial/\partial x_i$ when the context is clear, and denote by $\nabla := [\partial_1, \cdots, \partial_n]^\top \in \R^n$ the standard vector differential operator. Furthermore, we use the notation $u \cdot v$ to represent inner products, both between vectors and between differential operators and vector. We can then write the Hessian operator as $\nabla^2 = \nabla \nabla^\top$ and define the Laplace operator by $\Delta := \nabla \cdot \nabla = \Tr(\nabla^2) = \sum_{i=1}^n \partial_i^2$.

To gain intuition and especially on the derivatives of $\jfunc(t)$ at $t=0+$, we first perform an informal analysis using formal power series in what follows, before studying properties of the Bayes training error function $\tfunc(t)$ at $t=0+$ with full mathematical rigor. Let the formal power series ring $\R[[\sigma]]$ be the completion of the polynomial ring $\R[\sigma]$, we will write $\wt{o}(\sigma^k)$ to represent a power series with the leading order strictly larger than $k$, e.g. $\sum_{l=k+1}^\infty \sigma^l = \wt{o}(\sigma^k)$. We have:
\begin{proposition} \label{prop:bayes-train-informal}
Assume arbitrary differentiability of the density $p > 0$ of $\pi'$, and we can always interchange differentiation with integration. Consider the following formal power series at any $X,y$ and $\sigma$ for $u_k \in \R^n$ and $\alpha_k \in \R$,
\begin{align*}
    \frac{1}{\sqrt{n}} \prn{X\bysest - y} & = \sum_{k=0}^{+\infty} u_k \sigma^k, \qquad \trerr(\sigma^2) = \sum_{k=0}^{+\infty} \alpha_k \sigma^{k},
\end{align*}
whose coordinates and value belong to the polynomial ring $\R[\sigma]$. We have for the leading terms,
\begin{align*}
    \frac{1}{\sqrt{n}} \prn{X\bysest - y} & = \frac{\sigma^2 \nabla \log p(y)}{\sqrt{n}} + \frac{\sigma^4}{2\sqrt{n}} \prn{\nabla (\Delta \log p(y)) + 2\nabla^2 \log p(y) \nabla \log p(y) } + \wt{o}(\sigma^4), \\
    \trerr(\sigma^2) & = \frac{\sigma^4}{n} \cdot \E_{y \sim \pi'} \brk{\norm{ \nabla \log p(y) }^2} -\frac{\sigma^6}{n} \cdot \E_{y \sim \pi'} \brk{\norm{ \nabla^2 \log p(y) }_F^2} + \wt{o}(\sigma^6).
\end{align*}
\end{proposition}
See the proof in Appendix~\ref{proof:bayes-train-informal}. We emphasize that the $\wt{o}(\cdot)$ terms in Proposition~\ref{prop:bayes-train-informal} do not yield explicit rates of convergence nor any finite radius of convergence in $\sigma$. But rather, they encapsulate only higher-order terms in a formal power series expansion. While this perspective is conceptually informative, it does not provide rigorous analytical guaranties. The central idea underlying the above informal results is to expand $p(y) = \exp\{\log p(y)\}$ into its Taylor series. However, deriving explicit convergence rates, justifying the interchange of integration and differentiation, and determining the radius of convergence all require substantially more sophisticated mathematical tools.

By employing powerful techniques such as the Bakry–\'{E}mery calculus, we rigorously establish the following characterizations of $\tfunc$ up to $o(\sigma^6)$ (distinct from the previous term $\wt{o}(\sigma^6))$), when $\sigma^2 \to 0+$, validating the informal calculations in Proposition~\ref{prop:bayes-train-informal}. The proof is provided in Appendix~\ref{proof:bayes-train-asymptotics}.

\begin{theorem} \label{thm:bayes-train-asymptotics}
    Let $\pi'$ be the probability distribution of $X\theta$ induced by $\theta \sim \pi$.  Suppose that $\pi'$ has a density $p > 0$ on $\R^n$ and $ p$ is twice differentiable. If $\pi'$ has finite Fisher information $\jfunc_{\pi} = \jfunc(0) < +\infty$, the following quantities are well-defined for sufficiently small $t > 0$ and the limit holds
        \begin{align*}
            \lim_{t \to 0+} \frac{1}{n} \E_{y \sim \pi_t'} \brk{\norm{\nabla^2 \log p_t(y)}_F^2} = \frac{1}{n} \E_{y \sim \pi'} \brk{\norm{\nabla^2 \log p(y)}_F^2} =: \prn{-\jfunc_{\pi}'} < +\infty.
        \end{align*}
        then for sufficiently small $t > 0$,
        \begin{align*}
            \jfunc'(t) = - \frac{1}{n} \E_{y \sim \pi_t'} \brk{\norm{\nabla^2 \log p_t(y)}_F^2},
        \end{align*}
        and $\jfunc'(0)$ exists and is equal to $\jfunc_{\pi}'$ above. When $\sigma^2 \to 0+$, $|\tfunc(\sigma^2) - \sigma^4 \jfunc_{\pi} - \sigma^6 \jfunc_{\pi}'| = o(\sigma^6)$, i.e.\
        \begin{align*}
            & \left|\tfunc(\sigma^2) - \frac{\sigma^4}{n} \cdot \E_{y \sim \pi'} \brk{\norm{\nabla \log p(y)}^2} + \frac{\sigma^6}{n} \cdot \E_{y \sim \pi'} \brk{\norm{\nabla^2 \log p(y)}_F^2} \right|  = o(\sigma^6).
        \end{align*}
\end{theorem}
As a final remark concluding Theorem~\ref{thm:bayes-train-asymptotics}, we provide in the following lemma a sufficient condition under which the convergence $\jfunc'(t) \to \jfunc_\pi'$ holds. See its proof in Appendix~\ref{proof:regularity-higher-order}.
\begin{lemma} \label{lem:regularity-higher-order}
    Let $p>0$ be twice differentiable. If $\nabla p, \nabla^2 p$ are integrable, and
    \begin{align*}
        \Phi(y) := \frac{\norm{\nabla^2 p(y)}_F^2}{p(y)}, \qquad \text{and} \qquad \Psi(y) := \frac{\norm{\nabla p(y)}^4}{p(y)^3}
    \end{align*}
    belong to $\mc{L}^{1+\delta}(\R^n)$ for some $\delta > 0$. Then $\jfunc'(t)$ exists for all $t > 0$. In addition, $\jfunc_\pi'$ exists and is finite and $\jfunc'(t) \to \jfunc_\pi'$.
\end{lemma}

\subsection{Proof of Proposition~\ref{prop:bayes-train-informal}} \label{proof:bayes-train-informal}
To avoid heavy notations, we write $o_{\R[\sigma]}(\cdot)$ as $o(\cdot)$ in this proof, which subsumes higher-order terms in the polynomial ring $\R[\sigma]$.
\paragraph{Step I: Formal power series conditional on $y$.} We first consider expanding the normalized residual term for fixed $X, y$ in $\R^n$. By a variable transformation $z = (X\theta - y)/\sigma$,
\begin{align}
    \frac{1}{\sqrt{n}} \prn{X\bysest - y} & = \frac{1}{\sqrt{n}}\E \brk{X\theta -y \mid X, y} = \frac{1}{\sqrt{n}}\frac{\int (X\theta-y) \exp \brc{-\frac{\norm{X \theta - y}^2}{2 \sigma^2}} \pi(\de \theta) }{\int \exp \brc{-\frac{\norm{X \theta - y}^2}{2 \sigma^2}} \pi(\de \theta) } \nonumber \\
    & = \frac{\sigma}{\sqrt{n}}\frac{\int z \exp \brc{-\frac{\norm{z}^2}{2}}p(y+\sigma z) \de z }{\int \exp \brc{-\frac{\norm{z}^2}{2}} p(y + \sigma z) \de z } = \frac{\sigma}{\sqrt{n}}\frac{\int z \phi_1(z) p(y+\sigma z) \de z }{\int \phi_1(z) p(y + \sigma z) \de z }. \label{eq:mid-informal-1}
\end{align}
Let $\ell(y) = \log p(y)$, and define
\begin{align*}
    g = \nabla \ell(y) \in \R^n, \qquad H = \nabla^2 \ell(y) \in \R^2, \qquad T = \nabla^3 \ell(y) \in \R^3,
\end{align*}
where $T$ is the symmetric tensor for the third-order derivatives $T_{ijk} = \partial_i\partial_j\partial_k \ell(y)$. With the tensor-vector product notation $T[z,z,z] := \sum_{i,j,k} T_{ijk} z_iz_jz_k$, we can then explicitly write out the expansion of $p$ in its leading terms
\begin{align*}
    & p(y+\sigma z)  = \exp \brc{\log p(y+\sigma z)} = \exp \brc{\ell(y) + \sigma g \cdot z + \frac{\sigma^2}{2} z^\top H z +  \frac{\sigma^3}{6} T[z,z,z]  +  o(\sigma^3)} \nonumber \\
    & = p(y) \cdot \brc{1 + \sigma g \cdot z + \frac{\sigma^2}{2} \prn{z^\top H z  + (g\cdot z)^2}  + \frac{\sigma^3}{6} \prn{T[z, z, z]+ 3(g \cdot z) \cdot z^{\top} H z + (g \cdot z)^3} + o(\sigma^3)}.
\end{align*}
Since $\phi_1(z) \de z$ is the density for $\normal(0, I_n)$, to substitute the above display into Eq.~\eqref{eq:mid-informal-1} requires computation of moments the standard Gaussian vector $Z \sim \normal(0, I_n)$. Indeed, by Isserlis's formula in Lemma~\ref{lem:isserlis},
\begin{align*}
    \E[Z_iZ_jZ_kZ_l] = \ind_{i=j} \ind_{k=l} + \ind_{i=k} \ind_{j=l} + \ind_{i=l} \ind_{j=k},
\end{align*}
we can compute the following non-zero expectations (we omit the zero terms of odd-moments by symmetry),
\begin{align*}
    &\E \brk{Z (g \cdot Z)} = \nabla \ell(y), \qquad \E \brk{Z^\top H Z + (g \cdot Z)^2} = \Tr (H) + \norm{g}^2 = \Delta \ell(y) + \norm{\nabla \ell(y)}^2, \\
    &\E \brk{Z T[Z, Z, Z]} = 3 [\partial_1 \Delta \ell(y), \cdots, \partial_n \Delta \ell(y)]^\top = 3\nabla (\Delta \ell(y)), \\
    & \E \brk{(g \cdot Z) \cdot Z^\top HZ} = \Tr(H)g + 2Hg = \Delta \ell(y) \nabla \ell(y) + 2\nabla^2 \ell(y) \nabla \ell(y), \\
    & \E \brk{(g \cdot Z)^3} = 3 \norm{g}^2 g = 3 \norm{\nabla \ell(y)}^2 \nabla \ell(y).
\end{align*}
Combining the above displays with the expansion for $p(y+\sigma z)$ yields
\begin{align*}
    \frac{\E[p(y+\sigma Z)]}{p(y)} & =  1 + \frac{\sigma^2}{2} \prn{\Delta \ell (y) + \norm{\nabla \ell (y)}^2} + o(\sigma^3), \\
    \frac{\E[Zp(y+\sigma Z)]}{p(y)} & = \sigma \nabla \ell (y) + \frac{\sigma^3}{2} \prn{\nabla (\Delta \ell(y)) + \Delta \ell(y) \nabla \ell(y) + 2\nabla^2 \ell(y) \nabla \ell(y) + \norm{\nabla \ell(y)}^2 \nabla \ell(y) } + o(\sigma^3).
\end{align*}
Finally, identifying the preceding equations as the denominator and numerator for Eq.~\eqref{eq:mid-informal-1} we have
\begin{align*}
    &\frac{1}{\sqrt{n}} \prn{X\bysest - y} \\
    & = \frac{\sigma}{\sqrt{n}}\cdot \frac{\sigma \nabla \ell (y) + \frac{\sigma^3}{2} \prn{\nabla (\Delta \ell(y)) + \Delta \ell(y) \nabla \ell(y) + 2\nabla^2 \ell(y) \nabla \ell(y) + \norm{\nabla \ell(y)}^2 \nabla \ell(y) } + o(\sigma^3)}{1 + \frac{\sigma^2}{2} \prn{\Delta \ell (y) + \norm{\nabla \ell (y)}^2} + o(\sigma^3)} \nonumber \\
    & = \frac{\sigma^2}{\sqrt{n}}\cdot \brc{\nabla \ell(y) + \frac{\sigma^2}{2} \prn{\nabla (\Delta \ell(y)) + 2\nabla^2 \ell(y) \nabla \ell(y) } + o(\sigma^2)},
\end{align*}
where in the last equality we make use of $(1+cx)^{-1}=1-cx+o(x)$.
\paragraph{Step II: Computing $\trerr_X(\bysest)$ marginalizing over $y$.}

To compute the training error, it requires further using the expansion $y=y' + \sigma \tau$. Similarly we define the gradient, Hessian and third-order derivative tensor as $g'$, $H'$ and $T'$ at $y'$. We denote by $T'[\tau, \tau] \in \R^n$ whose $i$-th entry is $\sum_{j,k} T'_{ijk} \tau_j \tau_k$. Using those notations, we obtain the following.
\begin{align*}
    &\trerr(\sigma^2) \nonumber \\
    &= \frac{\sigma^4}{n} \cdot \E_{y' \sim \pi', \tau \sim \normal(0, I_n)} \brk{\norm{ g' + \sigma \cdot H' \tau + \frac{\sigma^2}{2} \prn{T'[\tau, \tau] + \nabla (\Delta \ell(y')) + 2 H' g' } + o(\sigma^2)}^2} \\
    & = \frac{\sigma^4}{n} \cdot  \E_{y' \sim \pi', \tau \sim \normal(0, I_n)} \brk{ \norm{ \nabla \ell (y') }^2 + \sigma^2 \prn{\norm{H'}_F^2 + g' \cdot T'[\tau, \tau]+ g' \cdot \nabla(\Delta \ell(y')) + 2 g'H'g'} + o(\sigma^2)},
\end{align*}
where we use $\E[\tau] = 0$ and $\E[\tau \tau^\top] = I_n$. To fully remove $\tau$, it still remains to calculate the expectation of $g' \cdot T'[\tau, \tau]$ over $\tau$ in what follows.
\begin{align*}
    \E_{\tau \sim \normal(0, I_n)} \brk{g' \cdot T'[\tau, \tau]} &= \sum_{i,j} g'_i T_{ijj} = g' \cdot \nabla(\Delta \ell(y')).
\end{align*}
Combined with integration by parts for $g'H'g'$,
\begin{align*}
    \E \brk{g'H'g'} & = \E \brk{\nabla \norm{\nabla \ell(y')}^2 \cdot \nabla \ell(y')} = -\E \brk{\Delta \norm{\nabla \ell(y')}^2} = - \E \brk{\Delta \norm{g'}^2},
\end{align*}
we have
\begin{align*}
    &\trerr(\sigma^2)  = \frac{\sigma^4}{n} \cdot  \E \brk{ \norm{ \nabla \ell (y') }^2 + \sigma^2 \prn{\norm{H'}_F^2 + 2g' \cdot \nabla(\Delta \ell(y')) - 2 \Delta \norm{g'}^2}}  + o(\sigma^6).
\end{align*}
In the last step, we use Bochner's identity~\cite[Prop.~9.2.2]{petersen2006riemannian} for Euclidean space
\begin{align*}
    \frac{1}{2} \Delta \norm{g'}^2 = g' \cdot \nabla(\Delta \ell(y')) + \norm{H'}_F^2,
\end{align*}
we conclude the proof with
\begin{align*}
    &\trerr(\sigma^2) = \frac{\sigma^4}{n} \cdot  \E \brk{ \norm{ \nabla \ell (y') }^2 - \sigma^2 \norm{H'}_F^2}  + o(\sigma^6).
\end{align*}

\subsection{Proof of Theorem~\ref{thm:bayes-train-asymptotics}} \label{proof:bayes-train-asymptotics}

By  Theorem~\ref{thm:Bayes_train_err}, we have $\jfunc(t) = \jfunc(0) + o(t)$. It remains to be shown that $\jfunc'(0) = \lim_{t \to 0+} \jfunc'(t)$---due to the smoothing by the Gaussian kernels, the differentiability of $\jfunc(t)$ when $t > 0$ is free, since we can always interchange integration and differentiation. However, the exact formula for $\jfunc'(t)$, is much more involved, requiring \citeauthor{bakry2006diffusions}'s calculus~\citep{bakry2006diffusions, bakry2013analysis}. We refer to the following explicit formula\footnote{We adapt to our scaling convention $t=\sigma^2$. The original form in the referenced paper is for the scaling $t = \sigma$.} from \cite[Eq.~(3.2)]{ledoux2016heat} as
\begin{align*}
    \jfunc'(t) = -\frac{1}{n} \E_{y \sim \pi_{t}'} \brk{\norm{\nabla^2 \log p_t(y)}_F^2}.
\end{align*}
We then use Theorem~\ref{thm:Bayes_train_err} to show the differentiability of $\jfunc(t)$ at $t=0+$. Indeed,
\begin{align*}
    \frac{\jfunc(t) - \jfunc(0)}{t} = \lim_{\epsilon \to 0+} \frac{\jfunc(t) - \jfunc(\epsilon)}{t} = \frac{1}{t} \int_{0+}^t \jfunc'(t) \de t. 
\end{align*}
Taking $t \to 0$, and by the assumption that $\lim_{t \to 0+}\jfunc'(t)$ exists, we conclude
\begin{align*}
    \jfunc'(0) = \lim_{t \to 0} \frac{1}{t} \int_{0+}^t \jfunc'(t) \de t = \lim_{t \to 0+} \prn{-\frac{1}{n} \E_{y \sim \pi_{t}'} \brk{\norm{\nabla^2 \log p_t(y)}_F^2}} = -\frac{1}{n} \E_{y \sim \pi'} \brk{\norm{\nabla^2 \log p(y)}_F^2}. 
\end{align*}
The proof is complete.

\subsection{Proof of Lemma~\ref{lem:regularity-higher-order}} \label{proof:regularity-higher-order}
Since $p, \nabla p, \nabla^2p$ are integrable, by Lebesgue differentiation theorem, one has almost surely
\begin{align*}
    p_t =p *\varphi_t \to p, \qquad \nabla p_t = \nabla p* \varphi_t \to \nabla p, \qquad \nabla^2p_t = \nabla^2 p * \varphi_t \to \nabla^2 p,
\end{align*}
where we use for any two differentiable functions $f$ and $g$ on the same space, $\nabla (f*g) = \nabla f* g= f*\nabla g$. Therefore, $p_t\nabla^2 \log p_t \to p\nabla^2 \log p$. To upgrade a.s.\ convergence to convergence in the mean, we want to invoke Vitali's convergence theorem, which requires uniform integrability certifications. We construct auxiliary functions
\begin{align*}
    \Phi_t(y) := \frac{\norm{\nabla^2 p_t(y)}_F^2}{p_t(y)}, \qquad \text{and} \qquad \Psi_t(y) := \frac{\norm{\nabla p_t(y)}^4}{p_t(y)^3},
\end{align*}
and we can write the upper bound
\begin{align}
    \norm{\nabla^2 \log p_t}_F^2 p_t = \norm{\frac{\nabla^2 p_t}{p_t}-\frac{\nabla p_t\left(\nabla p_t\right)^{\top}}{p_t^2}}_F^2 p_t \leq 2 (\Phi_t + \Psi_t). \label{eq:regularity-mid-1}
\end{align}
By H\"{o}lder's inequality, for any number $q > 0$ and functions $f$ and $g>0$,
\begin{align*}
    \norm{f(\cdot) \varphi_t(y-\cdot)}_{\mc{L}^1} \leq \norm{\frac{f(\cdot)}{g^{\frac{q}{q+1}}(\cdot)} \varphi_t^{\frac{1}{q+1}}(y-\cdot)}_{\mc{L}^{q+1}} \norm{g^{\frac{q}{q+1}}(\cdot) \varphi_t^{\frac{q}{q+1}}(y-\cdot)}_{\mc{L}^{q^{-1} + 1}}.
\end{align*}
Rearranging terms gives
\begin{align*}
    \frac{(f*\varphi_t)^{q+1}}{(g*\varphi_t)^q} \leq \prn{\frac{f^{q+1}}{g^q}}*\varphi_t.
\end{align*}
Applying the above inequality pointwise to $\Phi_t$ and $\Psi_t$ gives
\begin{align*}
    \Phi_t & = \frac{\norm{\nabla^2 p_t}_F^2}{p_t} = \frac{\norm{\nabla^2 p * \varphi_t}_F^2}{p * \varphi_t} \leq \Phi *\varphi_t, \\
    \Psi_t & = \frac{\norm{\nabla p_t}^4}{p_t^3} = \frac{\norm{\nabla p*\varphi_t}^4}{(p*\varphi_t)^3} \leq \Psi * \varphi_t.
\end{align*}
Return to Eq.~\eqref{eq:regularity-mid-1}, the above display then implies
\begin{align*}
    \norm{\norm{\nabla^2 \log p_t}_F^2 p_t}_{\mc{L}^{1+\delta}} \leq 2 \norm{\Phi + \Psi}_{1+\delta} \norm{\varphi_t}_1 = 2 \norm{\Phi + \Psi}_{1+\delta}.
\end{align*}
Hence, the functions $\{\norm{\nabla^2 \log p_t}_F^2 p_t\}$ indexed by $t$ are uniformly integrable in $\mc{L}^1$. By Vitali's convergence theorem~\citep[Thm.~16.14]{billingsley1995probability}, we conclude that
\begin{align*}
    \jfunc'(t) =-\int \norm{\nabla^2 \log p_t(y)}_F^2 p_t(y) \de y \to -\int \norm{\nabla^2 \log p(y)}_F^2 p(y) \de y = \jfunc_{\pi}'.
\end{align*}
The proof is complete.

\end{document}